\documentclass[nohyperref]{article}

\usepackage{microtype}
\usepackage{graphicx}
\usepackage{booktabs} 

\usepackage{hyperref}

\usepackage[accepted]{icml2022}

\usepackage{amsmath}
\usepackage{amssymb}
\usepackage{mathtools}
\usepackage{amsthm}

\usepackage{url}            
\usepackage{amsfonts}       
\usepackage{nicefrac}       
\usepackage{xcolor}         
\usepackage{comment}
\usepackage{mathrsfs}
\usepackage{subcaption}

\usepackage[capitalize,noabbrev]{cleveref}

\theoremstyle{plain}
\newtheorem{theorem}{Theorem}[section]

\theoremstyle{definition}

\theoremstyle{remark}

\newcommand*{\defeq}{\stackrel{\text{def}}{=}}

\usepackage[textsize=tiny]{todonotes}

\icmltitlerunning{Towards Practical Credit Assignment for Deep Reinforcement Learning}

\begin{document}

\twocolumn[
\icmltitle{Towards Practical Credit Assignment for Deep Reinforcement Learning}



\icmlsetsymbol{equal}{*}

\begin{icmlauthorlist}
\icmlauthor{Vyacheslav Alipov}{yandex}
\icmlauthor{Riley Simmons-Edler}{saicny}
\icmlauthor{Nikita Putintsev}{yandex}
\icmlauthor{Pavel Kalinin}{yandex}
\icmlauthor{Dmitry Vetrov}{nat,airi}
\end{icmlauthorlist}

\icmlaffiliation{yandex}{Yandex}
\icmlaffiliation{saicny}{Samsung AI Center NYC}
\icmlaffiliation{airi}{Artificial Intelligence Research Institute (AIRI)}
\icmlaffiliation{nat}{National Research University Higher School of Economics}

\icmlcorrespondingauthor{Vyacheslav Alipov}{\texttt{vyacheslav.alipov@gmail.com}}
\icmlcorrespondingauthor{Riley Simmons-Edler}{\texttt{rileys@cs.princeton.edu}}

\icmlkeywords{Machine Learning, Reinforcement Learning, Credit Assignment, Deep Reinforcement Learning}

\vskip 0.3in
]



\printAffiliationsAndNotice{}  

%


\begin{abstract}
\textit{Credit assignment} is a fundamental problem in reinforcement learning, the problem of measuring an action’s influence on future rewards. Explicit credit assignment methods have the potential to boost the performance of RL algorithms on many tasks, but thus far remain impractical for general use. Recently, a family of methods called Hindsight Credit Assignment (HCA) was proposed, which explicitly assign credit to actions in hindsight based on the probability of the action having led to an observed outcome. This approach has appealing properties, but remains a largely theoretical idea applicable to a limited set of tabular RL tasks. Moreover, it is unclear how to extend HCA to deep RL environments. In this work, we explore the use of HCA-style credit in a deep RL context.
We first describe the limitations of existing HCA algorithms in deep RL that lead to their poor performance or complete lack of training, then propose several theoretically-justified modifications to overcome them.
We explore the quantitative and qualitative effects of the resulting algorithm on the Arcade Learning Environment (ALE) benchmark, and observe that it improves performance over Advantage Actor-Critic (A2C) on many games where non-trivial credit assignment is necessary to achieve high scores and where hindsight probabilities can be accurately estimated.




\end{abstract}

\section{Introduction}
\label{sec:intro}

As evidenced by a number of high-profile examples in recent years, deep reinforcement learning has the potential to solve many real-world problems, from playing games \cite{alphago} to controlling robots \cite{openai2018dexterous} to regulating power usage in datacenters \cite{lazic2018data}. However, these success stories come at the price of extremely high data requirements and careful reward shaping to allow the RL agent to smoothly learn to perform the desired task, which is a prohibitive barrier for many other tasks of importance. 

One major contributing factor to this inefficiency is the difficulty of performing \textit{Credit Assignment}, the process of associating actions in the present with the rewards that they influence in the future. 
For example, as a human it can be hard to determine whether or not winning a baseball game was affected by what you ate that morning, or by the practice put in beforehand, or by wearing your lucky socks.
There are also examples among common RL benchmarks-- in the Atari game Seaquest the player needs to pick up and bring a diver to the surface to replenish their oxygen tanks within a time limit, and otherwise they will lose a life. A good credit assignment algorithm needs to associate picking up a diver and resurfacing with the additional points gained by using the extra oxygen to spend more time underwater.

While all RL algorithms can solve this sort of problem given sufficient training data, this can be prohibitively slow. For commonly used algorithms such as on-policy advantage actor-critic (A2C) \cite{mnih2016asynchronous} methods, it is difficult for the algorithm to discriminate between when actions are merely \textit{correlated} with future rewards and when they are \textit{causal}. This is because policy gradient algorithms such as A2C increase the probabilities of actions preceding a reward regardless of what effect those actions had on the probability of that reward being observed. In the limit this approach is guaranteed to converge to the optimal policy \cite{Sutton1998}, but if this prior assumption of correlation implying causation is inaccurate convergence may take a long time.

\subsection{Hindsight Credit Assignment}
\label{subsec:hca_intro}

Recently, \cite{harutyunyan2019hindsight} introduced Hindsight Credit Assignment (HCA), an algorithm for credit assignment. HCA uses information about future events to compute updates for the policy in hindsight. HCA only modifies the probabilities of actions that affect the likelihood of reaching rewarding states, and does not update actions that have no effect. 

This algorithm is theoretically well-motivated, and reduces the credit assignment problem to a supervised learning task -- to assign credit to a given action $a$, HCA learns a hindsight probability function $h_{\phi}(a| s_t, s_k) = P(a|s_t,s_k)$ given current state $s_t$ and future state $s_k$. This probability can be learned simply by minimizing the negative log-likelihood of actions given $s_t$ and $s_k$ pairs collected by a policy. It is then easy to weight updates to the policy based on how $h$ differs from the policy's foresight probability $\pi(a_t|s_t)$ -- the more the two probabilities differ, the bigger the effect the action had on the probability of observing the future state $s_k$, and the larger the corresponding policy update should be.

However, the existing HCA algorithm is still largely theoretical, has only been demonstrated on very simple tabular MDPs, and has recently been noted to have issues with certain types of MDPs \cite{mesnard2020counterfactual}. A number of practical questions remain -- How can we learn hindsight probabilities efficiently from limited data? How can we extend HCA to allow a value baseline? 

\subsection{Contributions}
\label{subsec:contributions}

In this paper, our goal is to derive practical deep RL credit assignment algorithms that build upon the hindsight credit assignment formalism. We describe several theoretically-justified extensions to HCA, and show that they allow for stable training in deep RL environments where HCA diverges or underperforms. We study the qualitative and quantitative behavior of the combined algorithm on games from the Arcade Learning Environment (ALE) and observe that it outperforms A2C on environments (like BeamRider) where assigning credit based on correlation is a poor prior.


We summarize our contributions as follows:
\begin{itemize}
    \item We probe the practical limitations of the existing hindsight credit assignment algorithm and the challenges of learning credit on common deep RL environments.
    \item We propose several extensions to HCA to address these issues and allow for improved training in deep RL environments.
    \item We characterize the credit assignment abilities of the resulting algorithm and validate their effects on the ALE benchmark
\end{itemize}

In the following sections, we first summarize other work on credit assignment in Section \ref{sec:related_work}, then explore the limitations of HCA and improve them in Section \ref{sec:methods}. We validate our algorithms in Section \ref{sec:experiments} on the ALE benchmark and conclude with some discussion of future research directions for credit assignment in Section \ref{sec:discussion}.

\section{Related Work}
\label{sec:related_work}

While the concept of credit assignment in RL is not new, an increasing amount of attention has been paid to \textit{explicit} credit assignment methods in recent years. While methods differ, the overall goal is to use additional information not normally used by the policy or value functions to compute more efficient policy updates. We will discuss some recent work on this topic here.

Counterfactual Credit Assignment \cite{mesnard2020counterfactual} performs credit assignment implicitly by providing the value function baseline with an additional vector which contains hindsight information about future states but avoids giving away information about the action taken to reach those states. CCA achieves unbiasedness by enforcing an information-theoretical bottleneck on the hindsight vector.
Zheng et al. \cite{zheng2021pairwise} propose a formulation of credit as a temporal weighting function of state, future state, and time horizon which is meta-learned. This approach uses meta-gradients to learn the credit function by differentiating through the policy gradient update. 
RUDDER \cite{arjona2018rudder} performs value transport to redistribute reward to preceding states which were important for reaching the rewarding state. This approach preserves the total magnitude of policy updates through redistribution across time, rather than eliminating updates for spurious actions.
Hung et al. \cite{hung2019optimizing} perform value transport by learning an attention model over past states, using it to redistribute reward similar to RUDDER. 
State Associative Learning \cite{raposo2021synthetic} also decomposes rewards into a sum over states by training a memory-augmented architecture with a carefully designed loss-function. 
Policy Gradients Incorporating the Future \cite{venuto2021policy} proposes a method that uses information from future states to improve policy gradients in the present, but does not try to solve the credit assignment problem directly.

Many recent online meta learning methods have conceptual similarity to credit assignment. Reward learning \cite{zheng2018learning} in particular has connections to value transport, as both involve a learned shaping of reward functions, but the quantities learned by meta-learning methods are open-ended, rather than a defined property of the environment with a specific formulation as in the case of explicit credit assignment. Xu et al. \cite{xu2018} proposed to learn hyperparameters for temporal credit assignment -- a discount $\gamma$ and bootstrapping parameter $\lambda$. 

\section{Methods}
\label{sec:methods}

Here, we describe the process and findings of our exploration of hindsight credit for deep RL. We begin with a review of RL notation and the credit formulation proposed in HCA \cite{harutyunyan2019hindsight}, which we base our work on, before discussing the shortcomings of this approach and our modifications that lead to a practical credit assignment algorithm.

\subsection{Background and Notation}
\label{sec:background}
Throughout this section we use capital letters for random variables, and lowercase letters for the values they take.

A Markov Decision Process (MDP) \cite{puterman2014markov} is a tuple $(\mathcal{S}, \mathcal{A}, p, r, \gamma)$. Given a current state $s\in\mathcal{S}$, an agent acting in the MDP takes an action $a\in\mathcal{A}$ and transitions to a next state $y\sim p(\cdot|s,a)$, receiving reward $r(s,a)$ in the process. The agent starts at an initial state $s_0$ and will act according to a policy $\pi$, i.e. $a\sim\pi(a|s)$, producing trajectories of states, actions, and rewards $\tau=(S_t,A_t,R_t)_t$, 
while seeking to maximize the expected discounted return $V^\pi(s_0)=\mathbb{E}_{\tau|s_0}[G_0]=\mathbb{E}\left[\sum_{k\geq0}\gamma^k R_k\right]$, where $G_t$ is the discounted return starting from state $s_t$.

While there exists a wide range of algorithms for deriving and improving a policy, in this work we are interested in policy gradient algorithms in particular, which improve a policy $\pi_\theta$ with parameters $\theta$ in the direction of the gradient of the value function $V^\pi$ \cite{sutton1999policy}:
\begin{equation}
    \label{eq:pg}
    \nabla_\theta V^{\pi_\theta}(s_0) = E_{\tau|s_0}\left[\sum_{t\geq0}\gamma^t\nabla_\theta\log\pi_\theta(A_t|S_t)G_t\right]
\end{equation}

Practical algorithms such as REINFORCE \cite{williams1992simple} approximate $G_t$ using an $T$-step truncated return $G_t\approx\sum_{k=t}^{t+T-1}\gamma^{k-t}R_k + \gamma^T V(S_{t+T})$.

\subsection{Hindsight Credit Assignment Preliminaries}
\label{subsec:hca_prelims}

HCA is a family of algorithms which modify the policy gradient update in Equation \eqref{eq:pg} through a credit function $C(a|s_t,F_t(\tau))$ that employs hindsight information $F_t(\tau)$, i.e. a concrete outcome from the trajectory $\tau$ following state $s_t$, to perform explicit credit assignment of obtained rewards to past actions.

While HCA proposes several possible quantities to use as $F_t(\tau)$, their \textit{State HCA} formulation, where $F_t(\tau)=s_k$ with $k > t$ is of particular interest as it allows for fine-grained credit assignment of each reward to each of the actions preceding it. In addition, \textit{State HCA} is able to assign credit to actions not actually taken, and can update them counter-factually. In this case, \cite{harutyunyan2019hindsight} propose the following update
 \begin{align}
     \label{eq:hca_update}
     &\nabla_\theta V^{\pi_\theta}(s_0) = E_{\tau|s_0}\left[ \sum_{t\geq0} \gamma^t\sum_a \nabla_\theta\log\pi_\theta(a|S_t)G^{C}_{t,a} \right] \nonumber\\
     &= E_{\tau|s_0}\left[ \sum_{t\geq0} \gamma^t\sum_a \nabla_\theta\log\pi_\theta(a|S_t) \left(\pi_\theta(a|S_t)r(S_t,a) \vphantom{\sum_{k>t}}\right. \right. \nonumber \\
     &\qquad \left. \left. + \sum_{k>t}\gamma^{k-t}C(a|S_t,S_k)R_k \right)\right],
 \end{align}
 where $C(a|s_t,s_k)=h(a|s_t,s_k)\defeq P(A_t=a|S_t=s_t,S_k=s_k)$. Intuitively $h(a|s_t,s_k)$ quantifies the relevance of action $a$ to the future state $s_k$ and thus to achieving future reward $R_k$. If $a$ is not relevant to reaching $s_k$ then $h(a|s_t,s_k)=\pi(a|s_t)$ since there is no additional information in $s_k$. If $a$ is instrumental in reaching $s_k$ then $h(a|s_t,s_k)>\pi(a|s_t)$ and, vice versa, if $a$ detracts from reaching $s_k$, $h(a|s_t,s_k)<\pi(a|s_t)$. We refer the reader to the reference paper \cite{harutyunyan2019hindsight} for additional intuition and formal analysis.
 
 Practical implementation of Equation \eqref{eq:hca_update} uses a $T$-step truncated return with bootstrapping:
 \begin{align}
     \label{eq:hca_approx_return}
     G^{C}_{t,a} \approx &\pi_\theta(a|S_t) \hat r (S_t, a) + \sum_{k=t+1}^{t+T-1} \gamma^{k-t}C(a|S_t,S_k)R_k \nonumber\\ 
     & + \gamma^{T}C(a|S_t,S_{t+T})V_\theta(S_{t+T}),
 \end{align}
 where $\hat r$ is an additional model trained to estimate immediate rewards, $V_\theta$ is an approximation of $V^{\pi_\theta}$, and $h_\phi(a|s_t,s_k)$ is a parametric model trained via cross-entropy to predict $A_t$.
 
 \cite{harutyunyan2019hindsight} prove that Equation \eqref{eq:hca_update} in the case where the credit function is perfectly accurate will be unbiased, and thus that the policy converges to the same optima as REINFORCE \cite{williams1992simple}. They further demonstrate that this algorithm converges faster than REINFORCE in several illustrative small tabular MDPs where credit assignment is critical for fast convergence.

\subsection{Deep Credit Assignment}
\label{subsec:deep_credit}

While the above HCA formulation has been shown to work in simple tabular environments \cite{harutyunyan2019hindsight} and on a set of illustrative problems \cite{mesnard2020counterfactual}, it's applicability to more complex MDPs in a deep RL framework remains unknown. Our objective in this section is to identify problems that arise when HCA is implemented using deep neural networks for function approximation (we refer to this version of the algorithm as Deep HCA) and to propose solutions to them.

For our experiments with Deep HCA, we used the Arcade Learning Environment (ALE) \cite{bellemare2013arcade} benchmark via the OpenAI Gym interface \cite{brockman2016openai}, as this benchmark is extremely common and thus represents the sort of tasks we want a practical algorithm to be able to handle. 

We describe our exploration of Deep HCA in the following sections. We first briefly describe Deep HCA implementation specifics and models in Section \ref{subsubsec:methods_impl}. Next, we identify some convergence problems for Deep HCA in Section \ref{subsubsec:methods_deep_hca_problems}, and propose algorithmic modifications to prevent them in Sections \ref{subsubsec:policy_prior} and \ref{subsubsec:advantage_credit}. 

\subsubsection{Deep HCA Models}
\label{subsubsec:methods_impl}

We implemented Deep HCA on top of a publicly available neural-network-based actor-critic baseline implementation \cite{pytorchrl} available under MIT licence to allow for direct comparisons between A2C and credit assignment methods.

We use the AtariCNN architecture first proposed by DQN \cite{mnih2015human} for the agent's policy $\pi_\theta(a|s)$, with an additional value head for $V^{\pi_\theta}$. To estimate hindsight probabilities, we train a separate AtariCNN network $h_{\phi}(a|s_t,s_k)$ which takes two concatenated states as input and predicts the action selected by the policy using softmax cross entropy.
    
In the original State HCA formulation in Equation \eqref{eq:hca_update} immediate reward is unconstrained by credit, which is sub-optimal for deep RL problems with delayed rewards, as immediate rewards may be unrelated to the current action. Further, in considering how the design of the State HCA algorithm translates when neural networks are used, the immediate reward model $\hat{r}$ in Equation \eqref{eq:hca_approx_return} sticks out as an added source of complexity. Thus, we propose the following simplified update rule that credits immediate reward using the credit function, eliminating the need for an additional reward model:
\begin{align}
    \nabla_\theta V^{\pi_\theta}(s_0)
    = E_{\tau|s_0}\left[ \sum_{t\geq0} \gamma^t\sum_a \nabla_\theta\log\pi_\theta(a|S_t) \right. \nonumber \\
        \qquad\left. \sum_{k\geq t}\gamma^{k-t}C(a|S_t,S_{k+1})R_k,
    \right],
    \label{eq:deep_hca_update}
\end{align}
where $C(a|s_t,s_k)=h_\phi(a|s_t,s_k)$.

This update arises if we consider reward as a function of the next state, rather than of the current state-action pair, or if the hindsight distribution is explicitly conditioned on reward. We further elaborate on this in Appendix \ref{app:credit_all_rewards}.

\subsubsection{Deep HCA Convergence Issues}
\label{subsubsec:methods_deep_hca_problems}

In our initial testing, we found that a straightforward implementation of Deep HCA barely makes any progress on ALE benchmark tasks. We identify the following reasons for Deep HCA's poor performance:
\begin{itemize}
    \item Slow training of the hindsight distribution $h_\phi(a|s_t,s_k)$ is a bottleneck for policy training.
    \item An imperfect and biased hindsight distribution $h_\phi(a|s_t,s_k)$ could lead to the agent's policy collapsing to a degenerate deterministic function in environments with negative rewards.
    \item The lack of advantage limits effective policy training due to high variance.
\end{itemize}

In the following sections we propose algorithmic modifications to speed up the training of the credit function and alleviate policy collapse while incorporating advantage to the State HCA formulation.

\subsubsection{Credit Approximation using a Policy Prior}
\label{subsubsec:policy_prior}

 From Equation \ref{eq:deep_hca_update} it follows that learning progress starts with the hindsight distribution $h_\phi(a|s_t,s_k)$. Before this classifier learns meaningful credit relationships no rewards are credited to any of the agent's actions, or worse rewards are credited wrongfully. 
 Because of this, $h_\phi$ needs to be able to rapidly adapt to novel states encountered by the agent, and to quickly credit rare significant outcomes. However, $h_\phi$ can't be trained too aggressively because overfitting leads to heavily biased updates. Thus, policy training is bottlenecked by $h_\phi$.

However, we notice that the hindsight distribution could be rewritten using Bayes' rule to explicitly use the agent's policy as a prior, i.e. $P(A_t=a|S_t=s_t,S_k=s_k)\propto P(S_k=s_k|S_t=s_t,A_t=a_t)\pi(a|s_t)$. We propose to model $h_\phi$ by explicitly taking this prior into account as $h_\phi(a|s_t,s_k)\propto\exp(g_\phi(a,s_t,s_k))\pi_\theta(a|s_t)$, which results in the following equation:
\begin{equation}
    \label{eq:policy_prior}
    h_\phi(a|s_t,s_k) = \frac{\exp({g_{\phi}(a,s_t,s_k) + \log\pi_{\theta}(a|s_t))}}{\sum\limits_{a'\in \mathcal{A}}\exp({g_{\phi}(a',s_t,s_k) + \log\pi_{\theta}(a'|s_t))}},
\end{equation}
where $g_{\phi}(a,s_t,s_k)$ is a learned credit residual on the policy logits.

This parametrization serves several purposes. First, it speeds up training, as the credit residual only needs to change when $s_k$ provides information that $s_t$ doesn't. Secondly, it allows $h_\phi$ to rapidly take changes in the agent's policy into account without the needing to re-learn the prior from data. Finally, it biases $h_\phi$ to be close to the agent's policy at the time of initialization and effectively avoids spurious crediting of rewards for $(a, s_t, s_k)$ tuples until $g_\phi$ learns a useful residual for them.

In Figure \ref{fig:classifier_ablation} we demonstrate that when trained in parallel on trajectories from an A2C agent on BeamRider, the explicit prior parametrization is able to incorporate future information and improve predictions early in training when the agent is close to random, which is critical as agent learning depends on the credit classifier.
This confirms that an explicit prior parametrization is better at adapting to novel trajectories as the agent trains.

\begin{figure}
    \centering
    \includegraphics[width=0.5\linewidth]{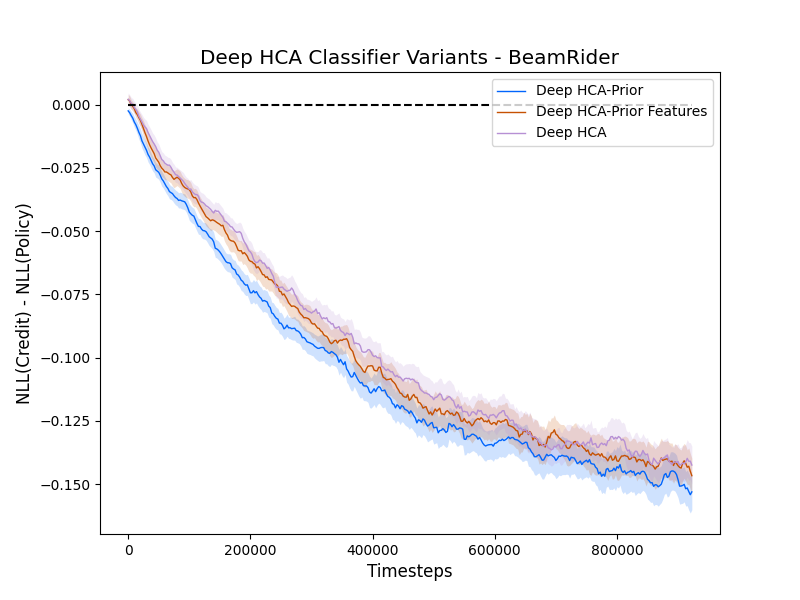}
    \caption{Comparison between different versions of the hindsight classifier $h_\phi(a|s_t,s_k)$ on BeamRider. The shaded region indicates the standard error across 15 training runs. Training curves for different variants of $h_\phi$ trained in parallel on trajectories sampled by an A2C agent have better NLL for predicting actions than the agent's policy.}
    \label{fig:classifier_ablation}
    \vspace{-.5cm}
\end{figure}

\subsubsection{Combining Credit and Advantage Via Bootstrapping}
\label{subsubsec:advantage_credit}

While testing the performance of Deep HCA, we discovered that State HCA suffers from policy collapse to a degenerate policy when returns can be negative (See training curves for Pong in Figure \ref{fig:hca_ablation}).

To illustrate this issue, let's consider the example of a scenario where the hindsight classifier is independent of future state, i.e. $h_\phi(a|s_t,s_k)=h_\phi(a|s_t)$, and doesn't change for a while. This could happen at the start of training or when encountering novel states. From Equation \ref{eq:deep_hca_update}, the update for state $s_t$ in this case would be $(\sum_{k\geq t}\gamma^{k-t}r_k)\sum_a h(a|s_t)\nabla_\theta\log\pi_\theta(a|s_t)=-G_t\nabla_\theta H(h_\phi,\pi_\theta)$, where $G_t$ is total return and $H(h_\phi,\pi_\theta)$ is the cross-entropy of the distribution $\pi_\theta$ relative to $h_\phi$. If $G_t$ is positive, this update minimizes cross-entropy, bringing $\pi_\theta$ closer to $h_\phi$. However, maximizing cross-entropy when $G_t$ is negative moves all the probability mass to $\arg\min_a h_\phi(a|s_t)$. We provide additional details on this phenomenon in Appendix \ref{app:neg_rewards}.

This property can be avoided if negative returns are transient throughout training for a given trajectory, as is the case for actor-critic algorithms. To achieve that in Deep HCA, we use a form of Potential Based Reward Shaping \cite{ng1999policy} with $V_\theta$ as a potential function, i.e. we substitute the original environment rewards $R_k$ with augmented rewards $\mathscr{A}_k = \gamma V_\theta(S_{k+1}) + R_k - V_\theta(S_k)$. This form of reward shaping doesn't change the optimal policy and introduces no additional bias while speeding up policy training and preventing the negative return collapse issue. The resulting alternative to the update rule in Equation \eqref{eq:deep_hca_update} is:
 \begin{align}
     \label{eq:deep_hca_adv_update}
     \nabla_\theta V^{\pi_\theta}(s_0)
     = E_{\tau|s_0}\left[ \sum_{t\geq0} \gamma^t\sum_a \nabla_\theta\log\pi_\theta(a|S_t) \right. \nonumber \\
     \qquad\left. \sum_{k\geq t}\gamma^{k-t}C(a|S_t,S_{k+1})\mathscr{A}_k
    \right],
 \end{align}
 where $T$-step truncated return approximation takes the simple form
 \begin{equation}
     \label{eq:deep_hca_adv_return}
     G^{C}_{t,a} \approx \sum_{k=t}^{t+T-1} \gamma^{k-t}C(a|S_t,S_{k+1})\mathscr{A}_k.
 \end{equation}
 Compared to Equation \eqref{eq:hca_approx_return} there is no need for a trailing value function, as the expectation of the augmented returns is zero. Note that the $V_\theta$ used as a potential function still approximates the discounted sum of original rewards $R_k$.
 
 It is worth noting that in addition to mitigating the policy collapse problem, these augmented rewards $\mathscr{A}_k$ are simply 1-step bootstrapped advantages. Using this algorithm combines advantages with the State HCA formalism and gains benefits such as improved training stability.
 We combine this modification with our policy prior, and call the resulting algorithm ``HCA-Value'' (and ``Deep HCA-Value'' respectively). We show its improved performance and training stability in Figure \ref{fig:hca_ablation}.

An interesting consequence of this formulation is that A2C is a special case of Equation \eqref{eq:deep_hca_adv_update}, where the credit term is $C(a|s_t,s_{k+1}) = [A_t = a]$ ($[]$ is the Iverson bracket and $A_t$ is a sampled action). That is, A2C simply credits the sampled actions and nothing else-- The truncated return in Equation \eqref{eq:deep_hca_adv_return} becomes a telescoping sum of 1-step bootstrapped advantages $\mathscr{A}_k$, which sums to a $T$-step advantage $\sum_{k=t}^{t+T-1}\gamma^{k-t}R_k + \gamma^T V_\theta(S_{t+T}) - V_\theta(S_t)$ for the selected action and to zero otherwise.

This means Deep HCA-Value is a generalization of A2C where the credit term is learned rather than fixed, allowing for more expressive power at the cost of needing to learn the credit distribution $h_{\phi}(s_t,s_k)$.

\subsubsection{Limiting the Effects of Counterfactual Updates}
\label{subsubsec:hindsight_clipping}


Empirically, we found that in some environments it's helpful to clip the hindsight probabilities $h_\phi$ based on the corresponding policy probabilities:
\begin{equation}
    \label{eq:hindsight_clipping}
    \Tilde{h}_\phi(a|s_t,s_k) = \min(h_\phi(a|s_t,s_k), \pi_\theta(a|s_t) \cdot \lambda),
\end{equation}
where $\lambda$ is a hyperparameter limiting how much credit an action can receive compared to the policy's current estimate of its importance.

In some environments clipping can stabilize training (see Breakout in Figure \ref{fig:hca_ablation}) and lead to significant performance gains (see NameThisGame and other examples in Appendix \ref{app:full_atari}), but it can be detrimental to performance when Deep HCA-Value already performs well. Finding a good value for $\lambda$ is thankfully intuitive -- increase it until performance starts to deteriorate. Empirically, we found that $\lambda=3$ worked well on all the games we tested.

As we mentioned in Section \ref{subsubsec:advantage_credit} Deep HCA-Value generalizes A2C so giving too much credit to selected action cannot cause convergence issues-- Deep HCA-Value performs the same update as A2C in such cases. This means that performance gains from clipping hindsight can only come from limiting the effects of counterfactual updates.

We call this variant of Deep HCA-Value using clipped $\Tilde{h}_\phi(a|s_t,s_k)$ ``Deep HCA-Value-Clip''.

\section{Experiments}
\label{sec:experiments}

In this section, we explore the credit assignment abilities of Deep HCA-Value and validate their effects on the ALE.

We elected to test a broad swathe of 32 ALE games including those which are unlikely to benefit much from better credit assignment.
We plot the performance of Deep HCA-Value on a selection of informative games in Figure \ref{fig:hca_ablation} which we believe reflect the range of typical behaviors, with full training curves for all 32 in Appendix \ref{app:full_atari}.

For each configuration we train 3 replicates for 50 million environmental steps (200M frames). All deep RL experiments were run on a common code base, hyperparameters shared between methods were tuned on A2C and kept the same for all methods, with hyperparameters specific to Deep HCA tuned via parameter sweeps on multiple ALE games. 
Our large-scale experiments were run on a cluster composed of 32 NVidia A100 GPUs over the course of about 2 weeks.

\begin{figure}
    \centering
    \includegraphics[width=0.49\linewidth]{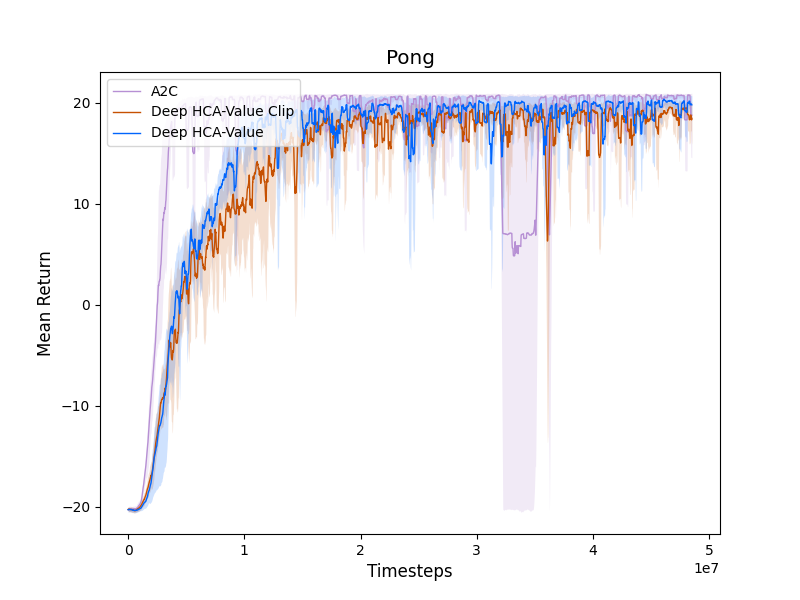} 
    \includegraphics[width=0.49\linewidth]{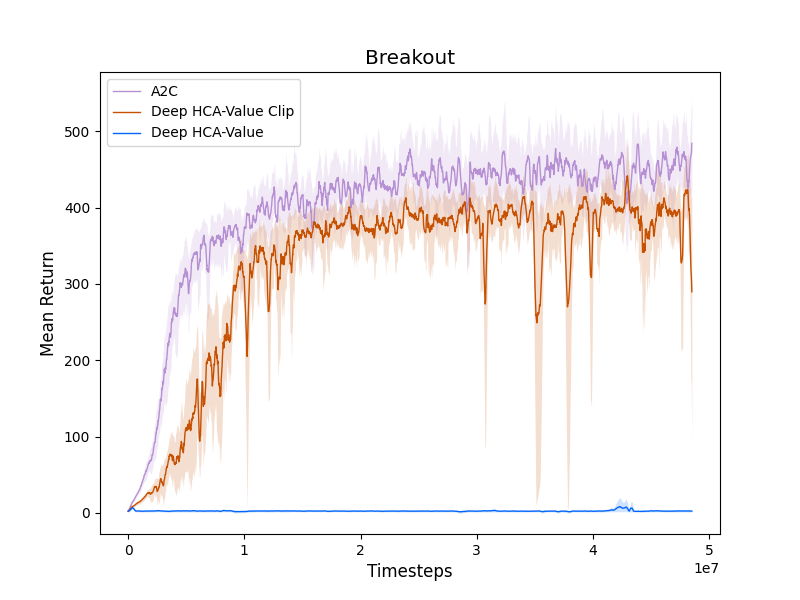} 
    \includegraphics[width=0.49\linewidth]{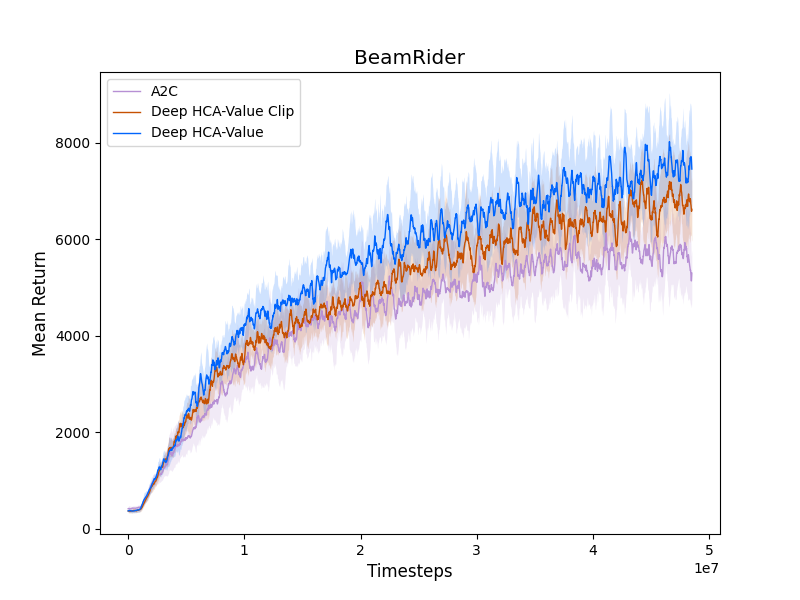} 
    \includegraphics[width=0.49\linewidth]{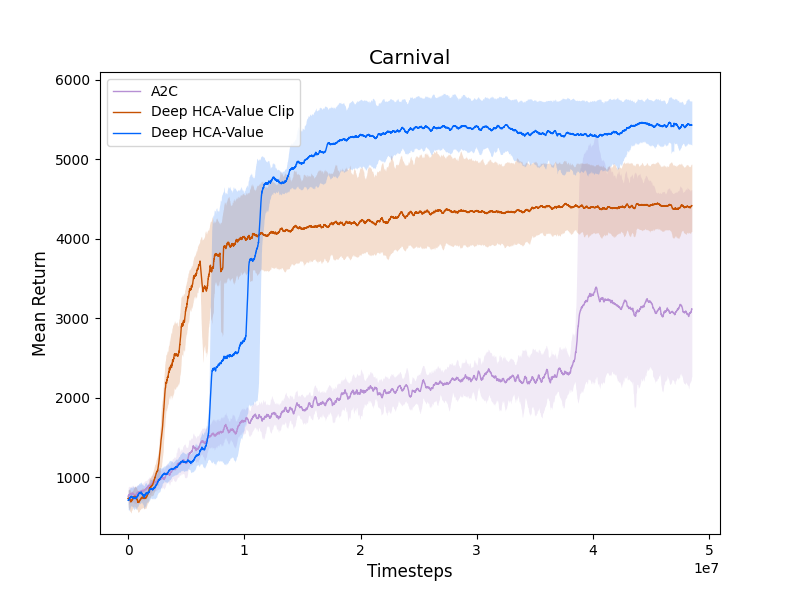}
    \includegraphics[width=0.49\linewidth]{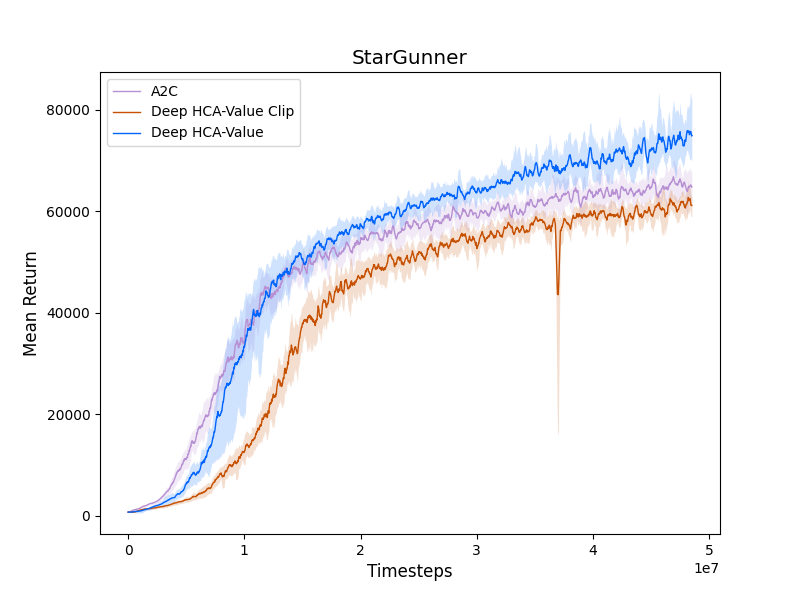}
    \includegraphics[width=0.49\linewidth]{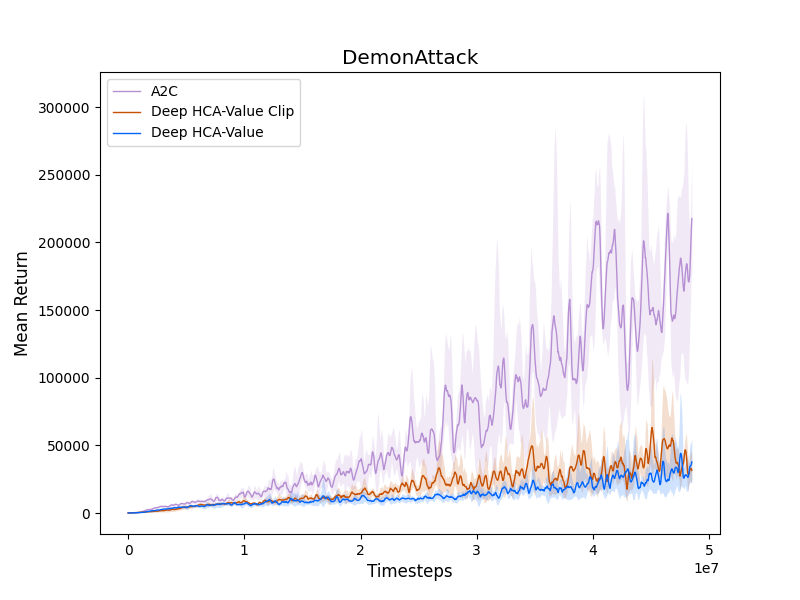}
    \caption{Training curves for our modifications to Deep HCA alongside reference algorithms. The shading shows min and max across 3 runs. Deep HCA-Value trains faster than A2C on BeamRider, Carnival, BeamRider and StarGunner but slower on DemonAttack and collapses on Breakout. Deep HCA-Value-Clip avoids collapse on Breakout).}
    \label{fig:hca_ablation}
    \vspace{-.8cm}
\end{figure}

From the results in Figure \ref{fig:hca_ablation} and Appendix \ref{app:full_atari}, we can see that Deep HCA-Value (with or without clipping) had better final performance than A2C on 10 environments, some (such as Carnival, Seaquest, NameThisGame, etc.) by large margins. In 8 games (Pong, Kangaroo, Amidar, etc.) Deep HCA-Value performed comparably to A2C. On a further 12 games Deep HCA-Value underperformed A2C, in some cases (such as DemonAttack) substantially. In 2 games (Asteroids and Atlantis) we did not observe Deep HCA-Value to improve its policy above random.

To understand these results, we explore what the credit classifier learned in environments where Deep HCA-Value overperforms or underperforms in Section \ref{subsec:what_has_credit_learned}. We further discuss the games that Deep HCA-Value performs well or poorly on and why in Appendix \ref{subsec:atari_results_explained}.

\subsection{What does the Credit Classifier Learn?}
\label{subsec:what_has_credit_learned}

As was mentioned concerning State HCA in Section \ref{subsec:hca_prelims}, Depending on the relative values of $h(a|s_t,s_{t+\Delta})$ and $\pi(a|s_t)$ there are three things credit assignment can do: 1) encourage an action when $h(a|s_t,s_{t+\Delta}) > \pi(a|s_t)$, 2) discourage an action when $h(a|s_t,s_{t+\Delta}) < \pi(a|s_t)$ and 3) make no update to an action when $h(a|s_t,s_{t+\Delta}) = \pi(a|s_t)$.

In order to take a broader look at Deep HCA-Value's behavior, we plotted the average negative log-likelihood (NLL) gains of the hindsight classifier $h_\phi(a|s_t,s_{t+\Delta})$ over the policy $\pi_\theta(a|s_t)$ when predicting the sampled action as a function of future time horizon $\Delta$, shown in Figure \ref{fig:beamrider_example}. Negative values indicate the credit classifier assigns higher probability to the sampled action than the policy (the first case), while positive values indicate lower probability (the second case), and 0 indicates equal probability (the third case). 


In Figure \ref{fig:nll_plots} we show NLL difference plots for four Atari games. Deep HCA-Value outperforms A2C on BeamRider and StarGunner, underperforms on DemonAttack, and fails to train without clipping on Breakout (see Figure \ref{fig:hca_ablation}). For these plots we use Deep HCA-Value Clip for Breakout and normal Deep HCA-Value for the others.

On average $h_\phi$ confidently guesses sampled actions that are close to the future state with a steep decline to almost random guessing after a certain horizon. This is unsurprising as in ALE the future is highly dependent on the nearest past and actions with long term consequences are rare.

In BeamRider, we notice that the effective credit horizon increases as the training progresses to a greater degree than in the other games. This means that as the agent converges and makes the future more dependent on its past actions the hindsight classifier is able to make use of further future states. This is in contrast with StarGunner, where credit assignment remains mostly short-term past early training.

In the two underperforming games, DemonAttack and Breakout, we notice a pattern of wide vertical lines where the hindsight classifier badly mispredicts sampled actions across all time horizons, causing wrong but confident counterfactual updates. This issue is particularly severe in Breakout at the start of the training, which leads to policy collapse for unclipped Deep HCA-Value that is prevented by clipping. We hypothesize that these severe mispredictions are caused by the classifier overfitting given limited and correlated data from the agent's rollouts on these games.

\begin{figure}
    \centering
    \hspace*{-1cm}\includegraphics[width=1.35\linewidth]{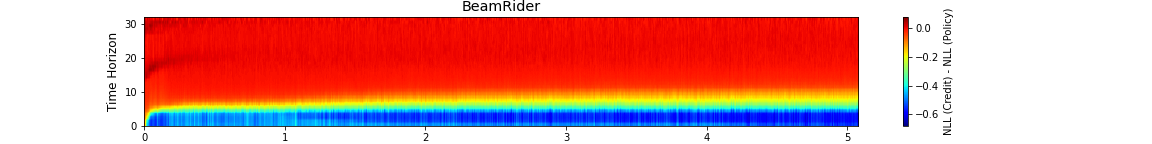}
    \hspace*{-1cm}\includegraphics[width=1.35\linewidth]{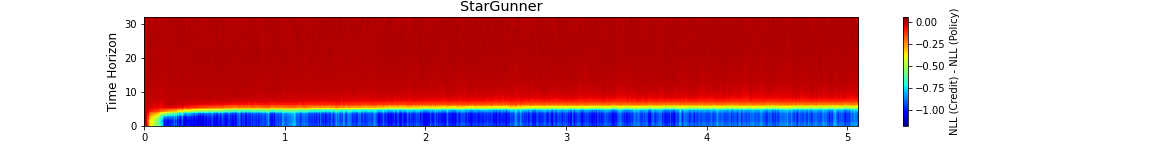}
    \hspace*{-1cm}\includegraphics[width=1.35\linewidth]{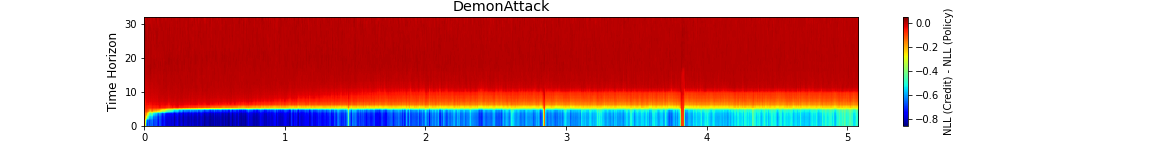}
    \hspace*{-1cm}\includegraphics[width=1.35\linewidth]{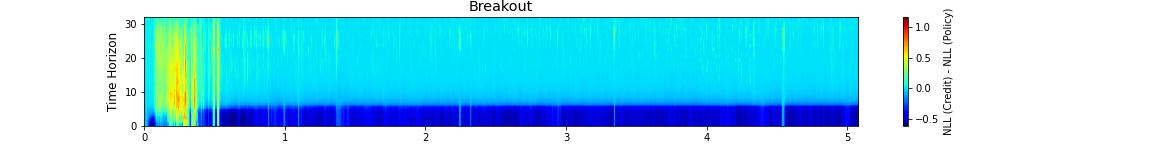}
    \caption{Average NLL difference between the credit classifier $h_\phi(a|s_t,s_{t+\Delta})$ and policy $\pi_\theta(a|s_t)$ when predicting sampled actions as a function of time horizon $\Delta$. The credit classifier predicts actions better given near-future states, and is close to the policy given far-future ones. In BeamRider the credit horizon grows during training, while in Breakout and DemonAttack severe mispredictions (indicated by vertical stripes of high NLL) hamper policy improvement.}
    \label{fig:nll_plots}
    \vspace{-.3cm}
\end{figure}


\subsubsection{Credit Versus N-Step Truncation}
\label{subsubsec:n_step}

Based on Figure \ref{fig:nll_plots}, the hindsight classifier $h_\phi(a|s_t,s_{t+\Delta})$ appears to mostly do two things: encourage sampled actions that are close to the rewarding state and give no credit to actions further away. Discouraging selected actions (and thus encouraging counterfactual actions) is relatively rare.

An important observation is that an extreme case of Deep HCA-Value is to encourage selected actions within some fixed time horizon $N$ with full weight, while giving no credit to actions outside that time horizon. Interestingly, this case is equivalent to A2C with N-step truncated advantages.

More formally, if we define the credit function as
$C(a|s_t,s_k) =
\begin{cases}
[A_t = a] & \text{if } k - t \leq N \\
\pi_\theta(a|s_t) & \text{otherwise}
\end{cases}$,
where $N$ is a fixed time horizon, then analogously to the connection to A2C in Section \ref{subsubsec:advantage_credit}, the truncated $T$-step return in Equation \eqref{eq:deep_hca_adv_return} becomes $[A_t = a]\sum_{i=t}^{t+N-1}\gamma^{i-t}\mathscr{A}_i
    + \pi_\theta(a|s_t)\sum_{k=t+N}^{t+N+T-1}\gamma^{k-t}\mathscr{A}_k$,
where the first term is the $N$-step advantage  (telescoping sum of 1-step advantages) for the selected action and zero otherwise, while the second term vanishes after plugging it into Equation \eqref{eq:deep_hca_adv_update} due to $\sum_a\pi_\theta(a|s)\nabla\log\pi_\theta(a|s)=0$.

To test whether the performance gains of Deep HCA-Value mostly stem from this extreme case we compare it against a variant of A2C that collects rollouts of the same length as Deep HCA-Value ($T=32$) but computes truncated advantages of length $N$ using a sliding window over trajectories. We use $N=5,10,15,20$ based on the results in Figure \ref{fig:nll_plots}.

In Figure \ref{fig:n_step} we observe that Deep HCA-Value outperforms all values of $N$ on BeamRider, suggesting that its advantages on this game stem from more sophisticated credit assignment-- either from counterfactuals or from learning to vary the effective rollout length dynamically. However on StarGunner Deep HCA-Value's result is matched by the $N=5$ configuration, confirming our intuition from Figure \ref{fig:nll_plots} that the credit classifier struggles to learn meaningful credits past a fixed time horizon for this game. We present comparisons on additional games in Appendix \ref{app:n_step_full}.

\begin{figure}
    \centering
    \includegraphics[width=0.49\linewidth]{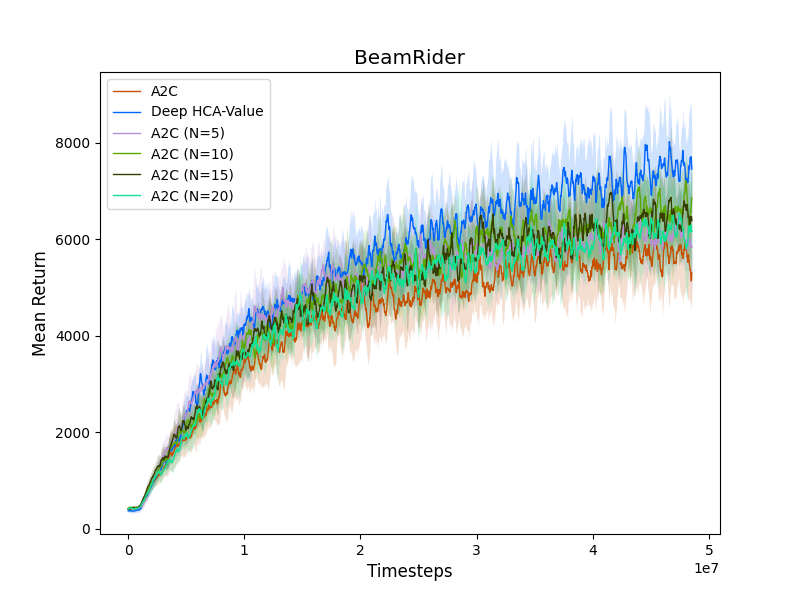} 
    \includegraphics[width=0.49\linewidth]{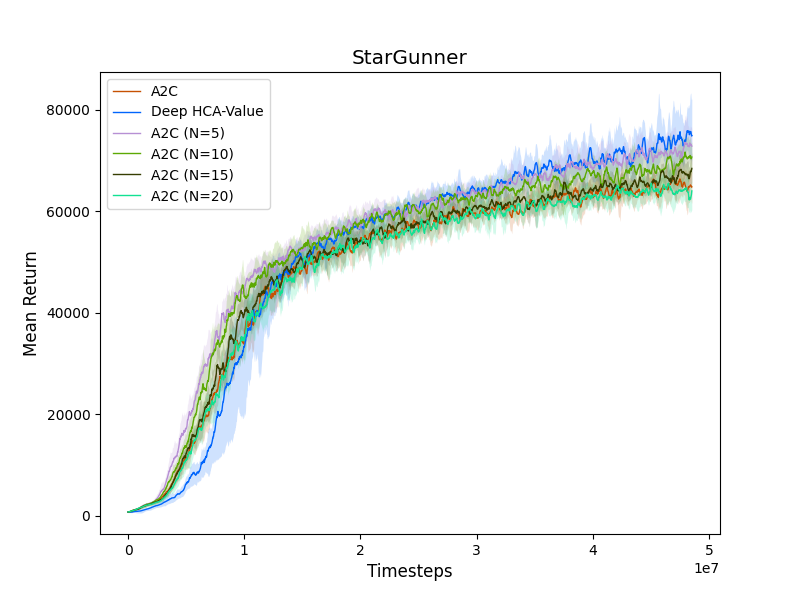}
    \caption{Comparison of Deep HCA-Value to A2C variants using $N$-step truncated advantages over trajectories of fixed length $T=32$. Deep HCA-Value outperforms all values of $N$ on BeamRider, 
    and matches $N = 5$ on StarGunner.
    }
    \label{fig:n_step}
    \vspace{-.5cm}
\end{figure}

\subsubsection{Qualitative Credit Examples}
\label{subsec:behavior}

In order to further confirm that Deep HCA-Value learns meaningful credits we took a closer look at the predictions of the credit classifier on BeamRider. In Figure \ref{fig:beamrider_example} we present predictions for $\pi_\theta(a|s_{t-\Delta})$ and $h_\phi(a|s_{t-\Delta},s_t)$ for rewarding state $s_t$ (hit) and 4 states $s_{t-\Delta}$ in the past for $\Delta=22,20,5,3$ presented in this order from left to right.

In state $s_{t-22}$ the agent decides to FIRE while the credit classifier counterfactually encourages standing still (NOOP and UP) and discourages any movement that decreases the odds of reaching $s_t$. Over the next 9 steps (not shown) the agent performed sporadic actions but returned to the same lane, which means that the Deep HCA-Value update in this case will shorten the path to the rewarding state $s_t$.

In state $s_{t-20}$ the agent is changing lanes, which takes several timesteps during which all actions are effectively NOOP. $h_\phi$ correctly gives no credit to any of the actions.

In state $s_{t-5}$ the agent is standing in the leftmost lane, and decides to LEFTFIRE. $h_\phi$ gives credit for the resulting hit to LEFTFIRE and to the equally good FIRE. It also discourages moving away from the target or doing nothing.

In state $s_{t-3}$ the classifier encourages every action that does not move. Though these actions do not cause reward, they increase the odds of reaching $s_t$ and so they're given credit. This is an undesirable crediting, but $h_\phi$ has no way of learning which features of the state are relevant to the task and which are not. In this case the exploding target is a relevant feature, while the spaceship's position is not.

\begin{figure}
    \centering
    \hspace*{0.40cm}\includegraphics[width=0.16\linewidth]{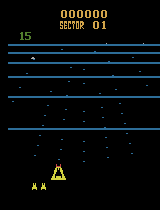}
    \hspace*{0.06cm}\includegraphics[width=0.155\linewidth]{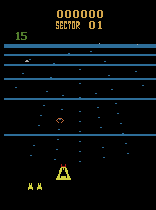}
    \hspace*{0.04cm}\includegraphics[width=0.16\linewidth]{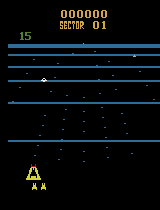}
    \hspace*{0.04cm}\includegraphics[width=0.16\linewidth]{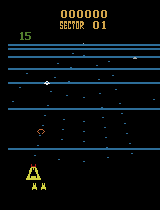}
    \hspace*{0.02cm}\includegraphics[width=0.16\linewidth]{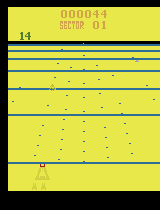}
    \hspace*{-1.4cm}\includegraphics[width=0.75\linewidth]{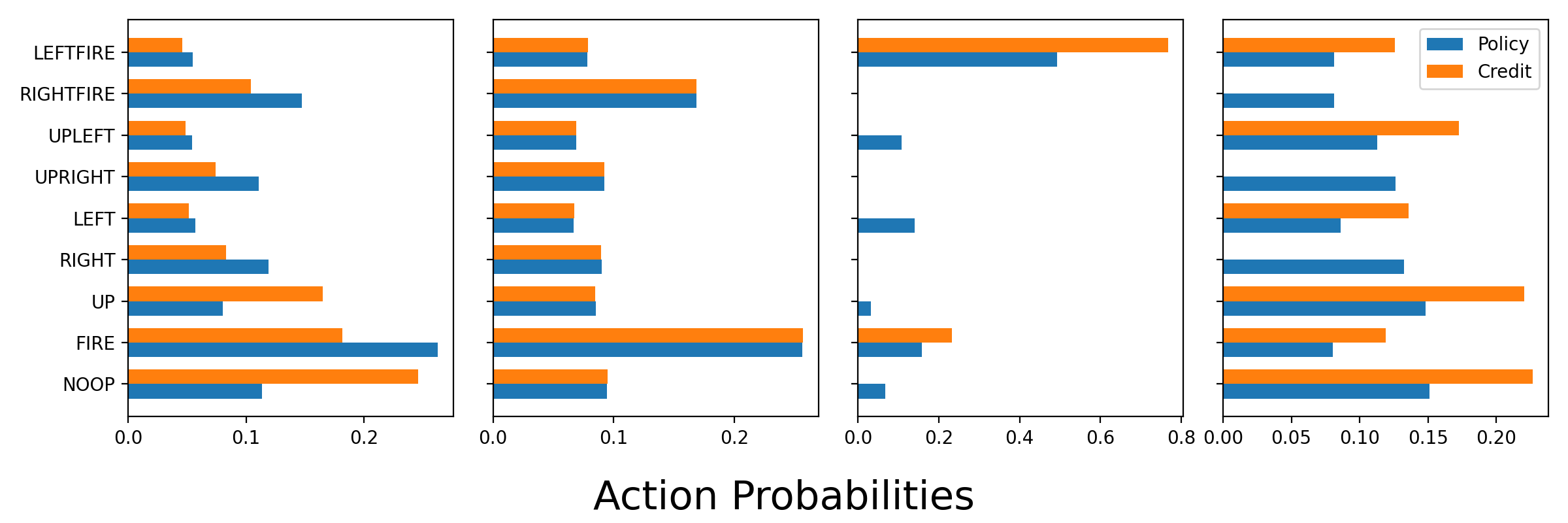}
    \caption{Example hindsight predictions on a sequence of frames leading up to a rewarding state in BeamRider.}
    \label{fig:beamrider_example}
    \vspace{-.3cm}
\end{figure}

\subsubsection{Behavioral Differences}
\label{subsec:behavior_diffs}
We looked at the learned behaviour of A2C and Deep HCA-Value on the game of Carnival to further explore their behavior. We observed in Figure \ref{fig:hca_ablation} that Deep HCA-Value outperforms A2C on this game, which motivates understanding.

Carnival is a shooting gallery. There are many small low-rewarding targets, and it is easy to hit them by shooting randomly. In addition, there is a highly-rewarding windmill which is almost entirely hidden by an orange block and therefore requires precise aim to hit. When there are no other targets on the screen, destroying the windmill ends the episode and the agent receives a bonus reward for the number of unused bullets. Example frames are in Figure \ref{fig:carnival_rollouts}.

As popularized by DQN \cite{mnih2015human}, we train with reward clipping, so both frequent low-rewarding hits on the small targets and rare high-rewarding hits on the windmill result in $+1$ reward. As a result, a hit on the windmill has a very small impact on the return, whereas it is easy to obtain distractor rewards. Thus, to learn to reliably hit the windmill an agent must perform efficient credit assignment.

As shown in the top row of Figure \ref{fig:carnival_rollouts}, A2C doesn't learn to hit the windmill, and spends all its bullets hitting small low-rewarding targets, while Deep HCA-Value speeds up reinforcement of the precision shots needed to hit the windmill, as seen in Figure \ref{fig:hca_ablation}.

We observed similar results for Seaquest and NameThisGame where Deep HCA-Value outperforms A2C by a large margin by identifying much more nuanced behavior.


\begin{figure}
    \centering
    \includegraphics[width=0.99\linewidth]{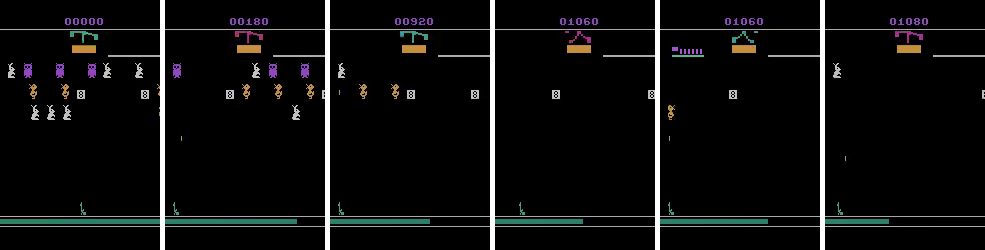} 
    \includegraphics[width=0.99\linewidth]{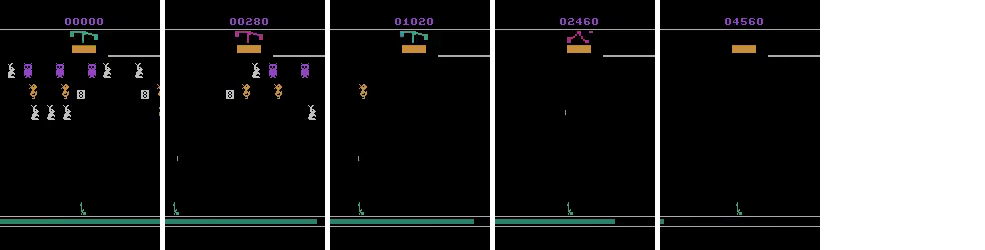}
    \caption{Rollouts from policies trained via A2C (top row) and Deep HCA-Value (bottom row) on Carnival. Screenshots are taken every ~2 seconds. Deep HCA-Value finds the highly rewarding behaviour of shooting at the windmill once other targets are gone,
    while A2C usually never leaves the left side of the screen.}
    \label{fig:carnival_rollouts}
    \vspace{-.3cm}
\end{figure}

\section{Discussion and Conclusions}
\label{sec:discussion}

Finally, we conclude with discussion. Our goal in this work was to explore the benefits, behaviors, and limitations of the HCA credit assignment framework on a relatively complex deep RL benchmark. Starting from the fundamentals of HCA, we developed Deep HCA in several variants, which improved performance on some Atari games. We also investigated what the credit classifier actually learns, and in the process discovered a number of interesting directions for further improving and understanding credit assignment.

The most direct of these is improving on the limitations of credit classifier training. While the Deep HCA-Value formulation can learn non-trivial credit assignment, performance on many games is limited by its stability and learning speed. 
We also lack a complete understanding as to what aspects of an environment influence the benefits of credit assignment. 

In the future, we hope to answer these questions, and further hope that this work can inspire and inform other research into practical credit assignment algorithms for deep reinforcement learning.

\bibliography{references.bib}
\bibliographystyle{icml2022}


\newpage
\appendix
\onecolumn

\section{Appendix}

\subsection{Tabular Experiments}
\label{app:tabular}
In addition to our ALE experiments, we also implemented tabular versions of State HCA and A2C to validate our modifications in a well-understood setting. As we are primarily concerned with the tabular validity of our modifications in this section, we used the FrozenLake benchmark \cite{brockman2016openai} as a representative tabular task which is not specifically designed to benefit from credit assignment (though credit does affect performance), but which allows us to see whether our modifications improve or worsen performance in a ``typical'' environment without the challenges of neural network function approximation.

The results of this validation are shown in Figure \ref{fig:tabular_exps}. We evaluated tabular versions of HCA, HCA-Prior, and HCA-Value, running each method with 100 random training seeds and plotting the mean and min/max of all runs. We observe that the policy prior improves over normal HCA, which is stable in this environment. This suggests that the policy prior is applicable to many types of MDPs, rather than being Atari specific. Similar to the results seen for Atari in Figure \ref{fig:hca_ablation}, HCA-Value improves upon HCA-Prior in turn.

To validate our finding that HCA without a value baseline is vulnerable to premature convergence if returns can be negative, we also tested our tabular HCA variants in a modified version of FrozenLake, where we apply a -1 reward penalty when the agent falls in the ``holes'' in the ice. In this case, the optimal policy is the same, but as we see in Figure \ref{fig:tabular_exps} HCA-Prior converges to a return of 0 (avoids holes, but does not cross the ice), while HCA-Value's performance is not reduced compared to the original task.

\begin{figure}
    \centering
    \includegraphics[width=0.49\linewidth]{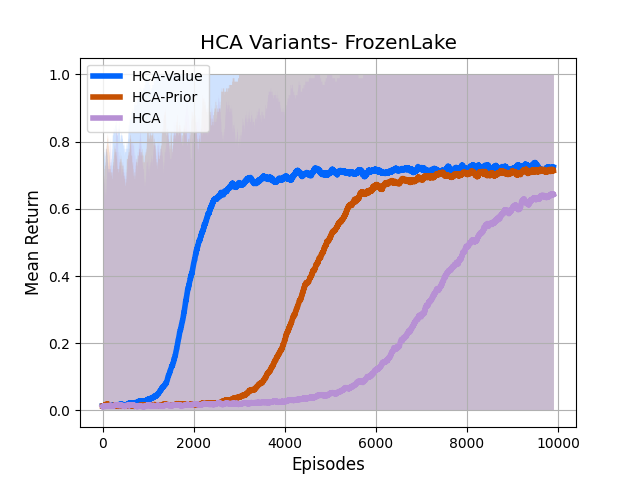} 
    \includegraphics[width=0.49\linewidth]{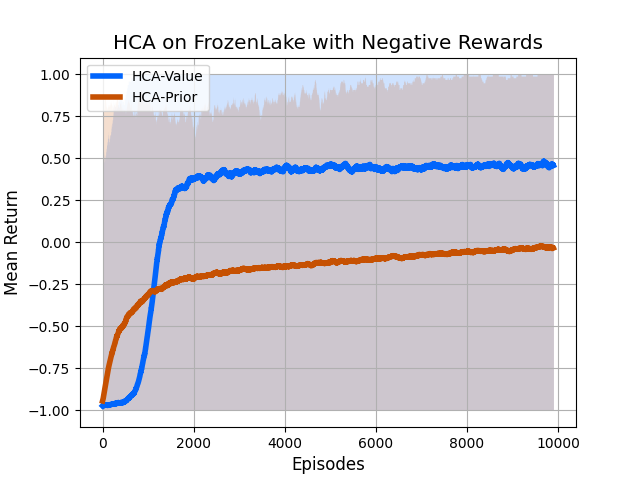}    
    \caption{Tabular validation of our HCA extensions on FrozenLake (left) and FrozenLake with negative rewards (right). HCA-Prior improves over normal State HCA, and HCA-Value improves over HCA-Prior. Like in the ALE, HCA-Prior converges to a local maxima when negative rewards are present.}    
    \label{fig:tabular_exps}
\end{figure}

\subsection{Full Atari Results}
\label{app:full_atari}
In Figure \ref{fig:atari_training_curves} we present the full training curves for each game we tested in the ALE. We plot the mean of three training runs for each method, and apply a rolling average of window size 100 for smoothing. The min and max among the three replicates are shown by the shaded region.


\begin{figure}
    \centering\vspace{-0.3cm}
    \includegraphics[width=0.24\linewidth]{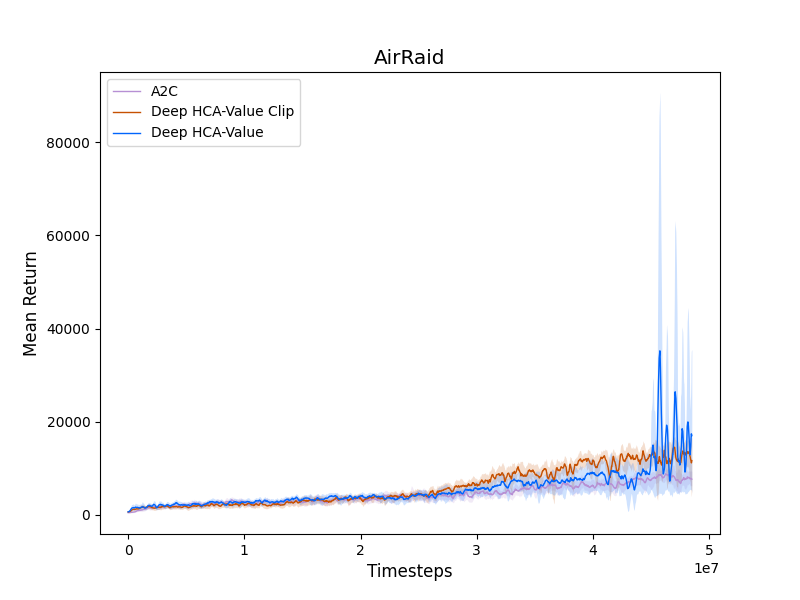}
    \includegraphics[width=0.24\linewidth]{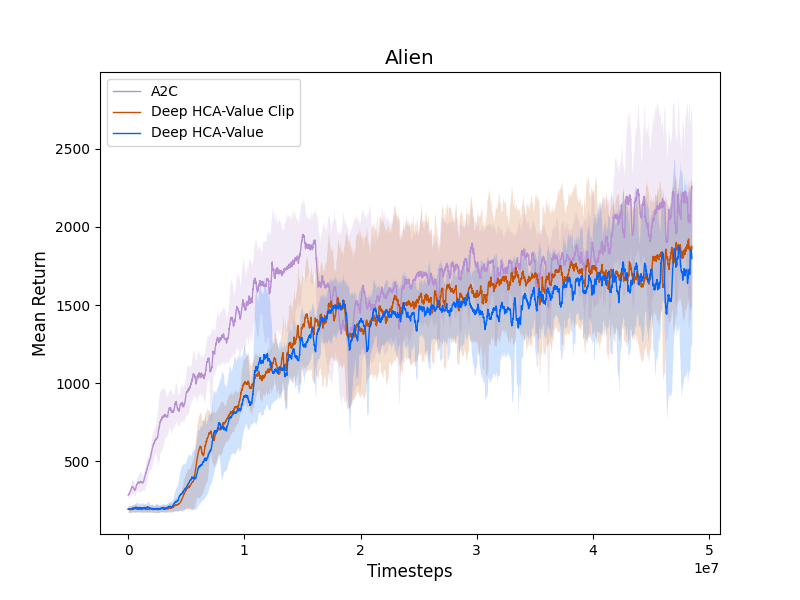}
    \includegraphics[width=0.24\linewidth]{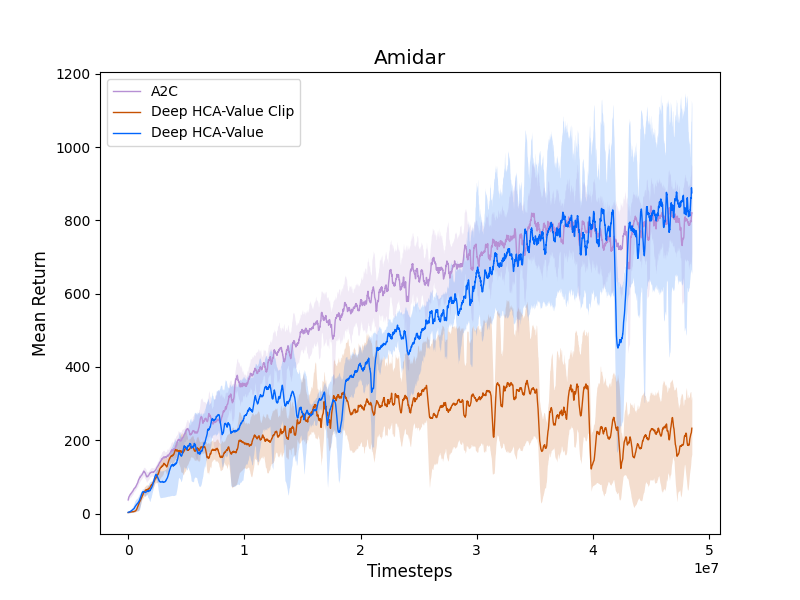}
    \includegraphics[width=0.24\linewidth]{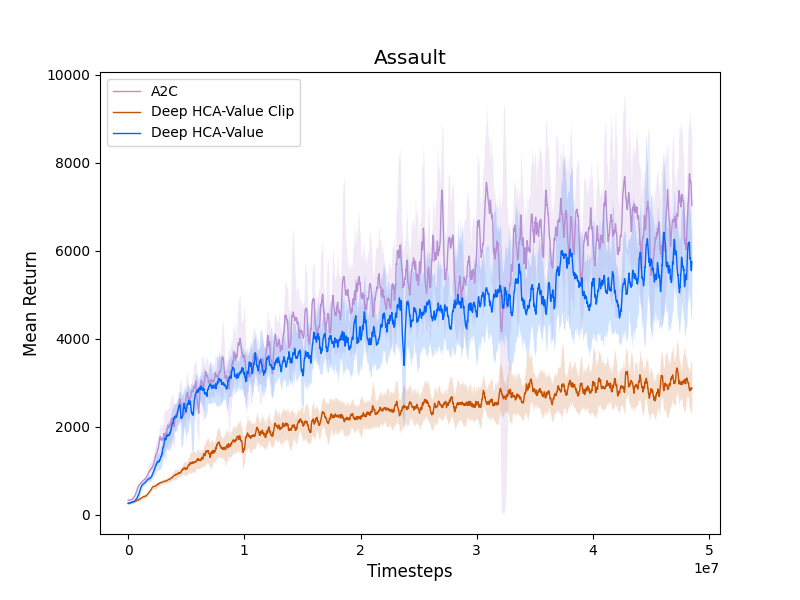}\vspace{-0.3cm}
    \includegraphics[width=0.24\linewidth]{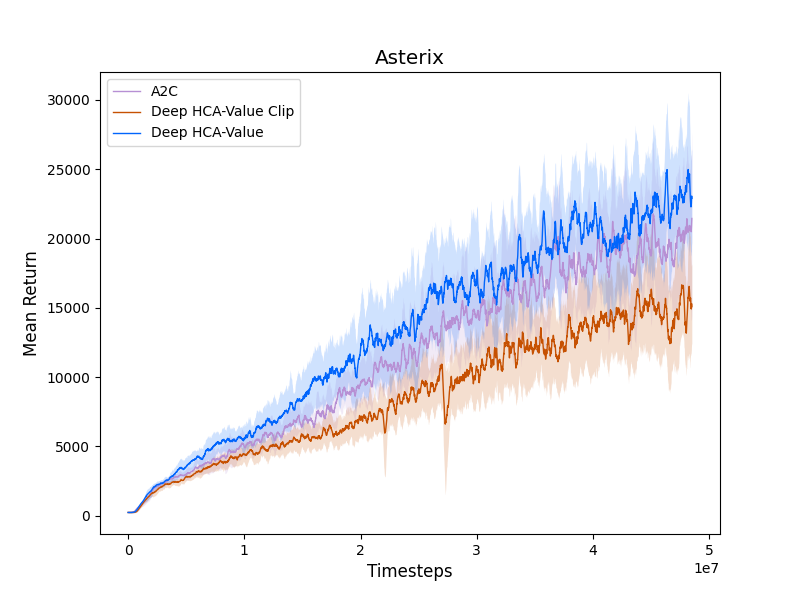}
    \includegraphics[width=0.24\linewidth]{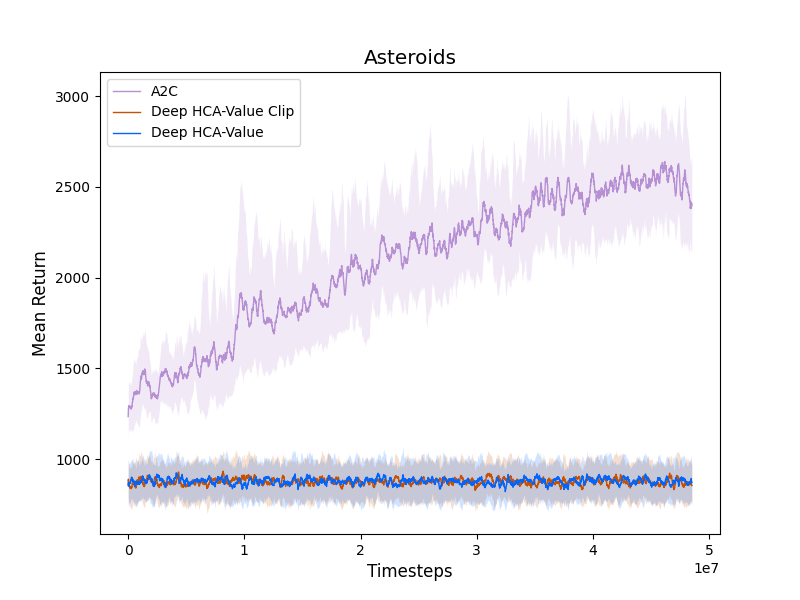}
    \includegraphics[width=0.24\linewidth]{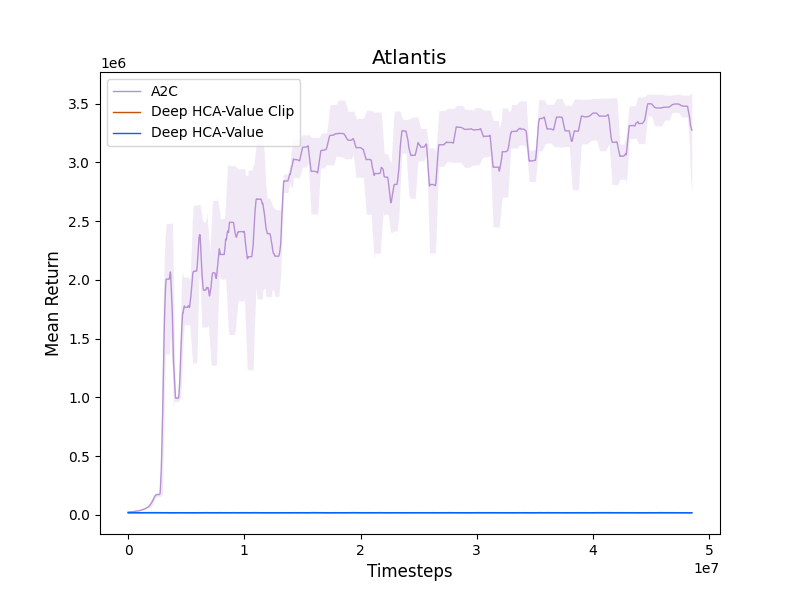}
    \includegraphics[width=0.24\linewidth]{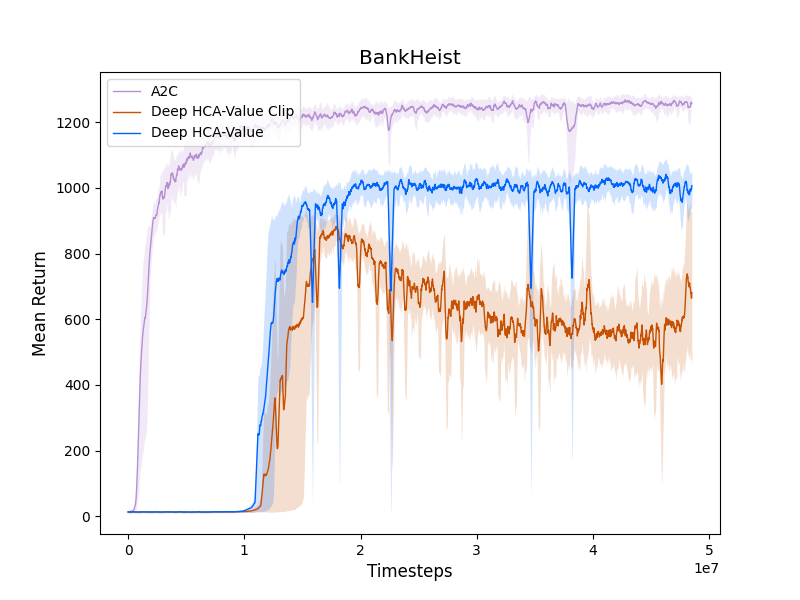}\vspace{-0.3cm}
    \includegraphics[width=0.24\linewidth]{images/icml_performance_plots/beamrider.png}
    \includegraphics[width=0.24\linewidth]{images/icml_performance_plots/breakout.png}
    \includegraphics[width=0.24\linewidth]{images/icml_performance_plots/carnival.png}
    \includegraphics[width=0.24\linewidth]{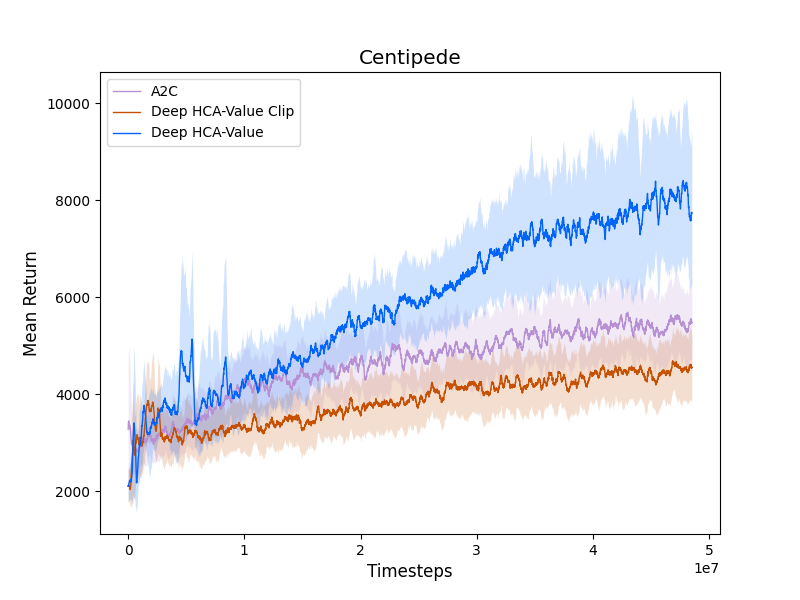}\vspace{-0.3cm}
    \includegraphics[width=0.24\linewidth]{images/icml_performance_plots/demonattack.png}
    \includegraphics[width=0.24\linewidth]{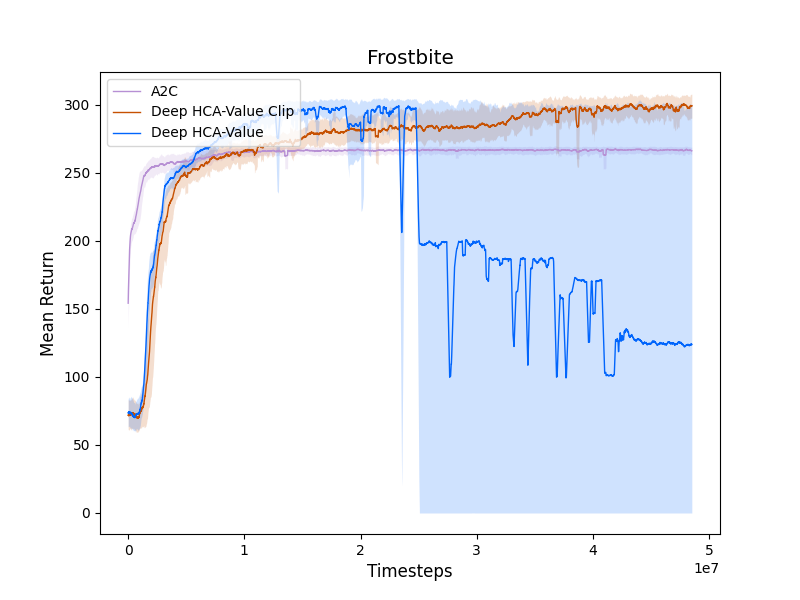}
    \includegraphics[width=0.24\linewidth]{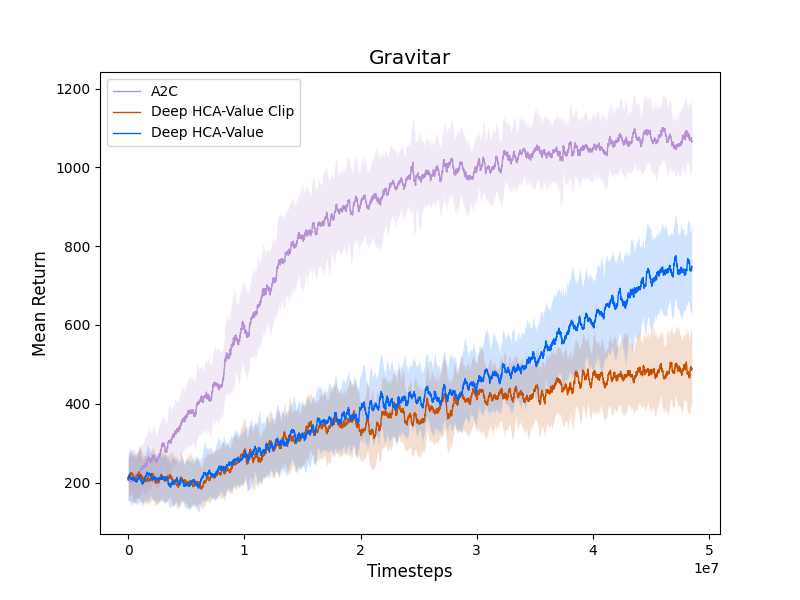}
    \includegraphics[width=0.24\linewidth]{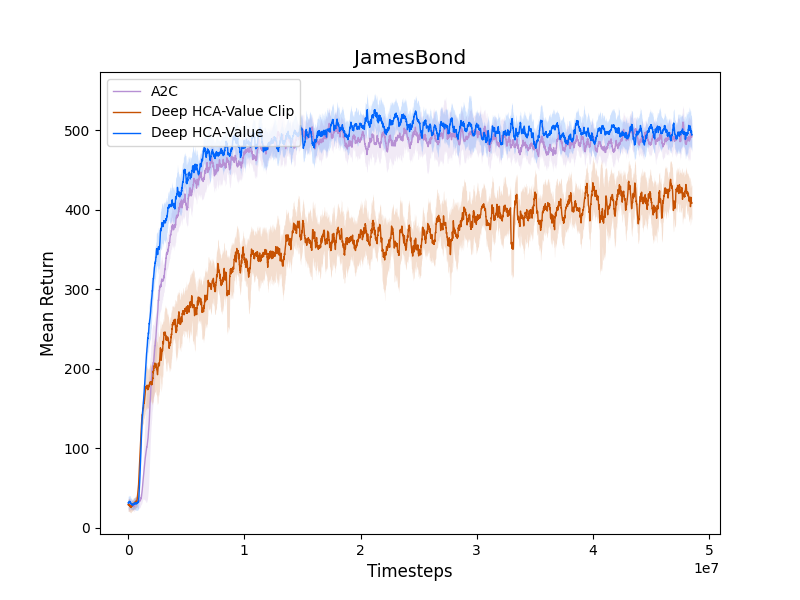}\vspace{-0.3cm}
    \includegraphics[width=0.24\linewidth]{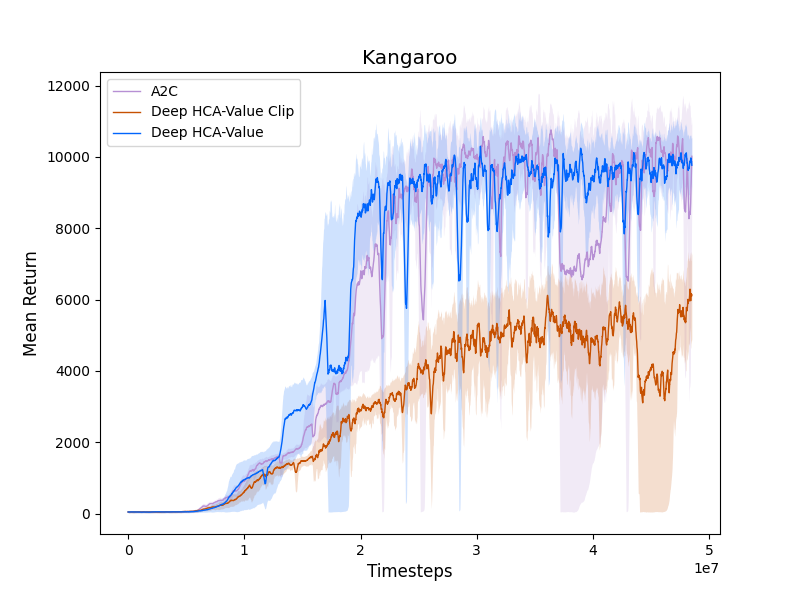}
    \includegraphics[width=0.24\linewidth]{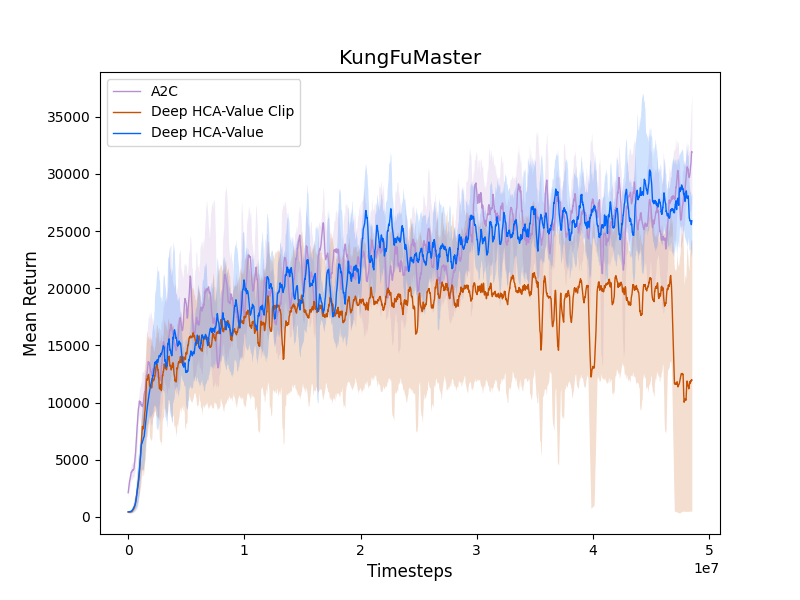}
    \includegraphics[width=0.24\linewidth]{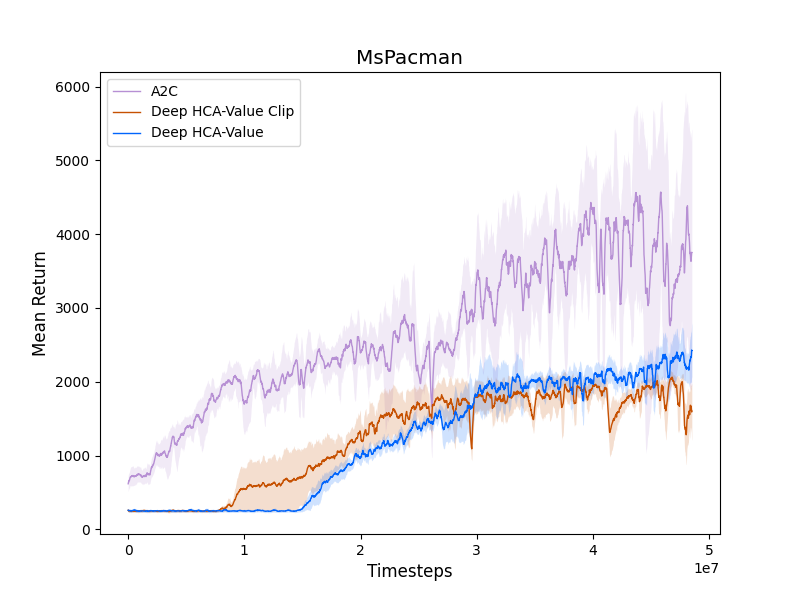}
    \includegraphics[width=0.24\linewidth]{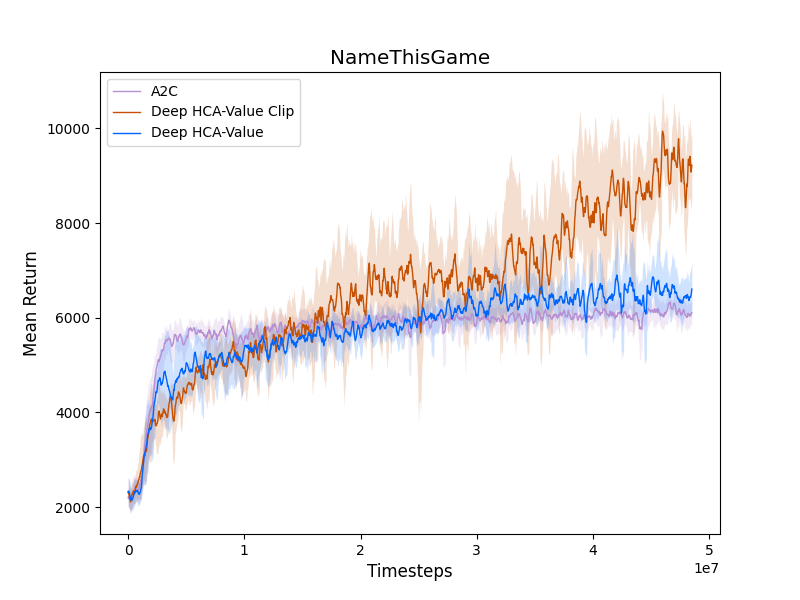}\vspace{-0.3cm}
    \includegraphics[width=0.24\linewidth]{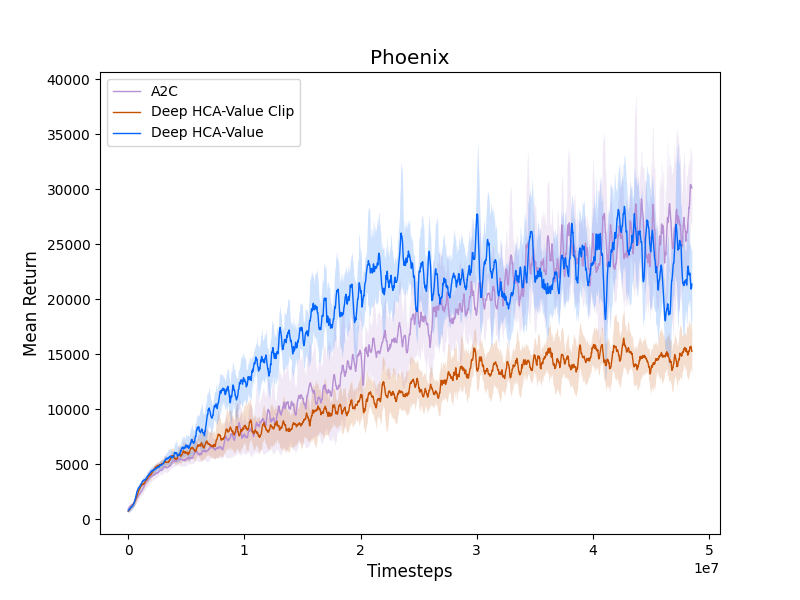}
    \includegraphics[width=0.24\linewidth]{images/icml_performance_plots/pong.png}
    \includegraphics[width=0.24\linewidth]{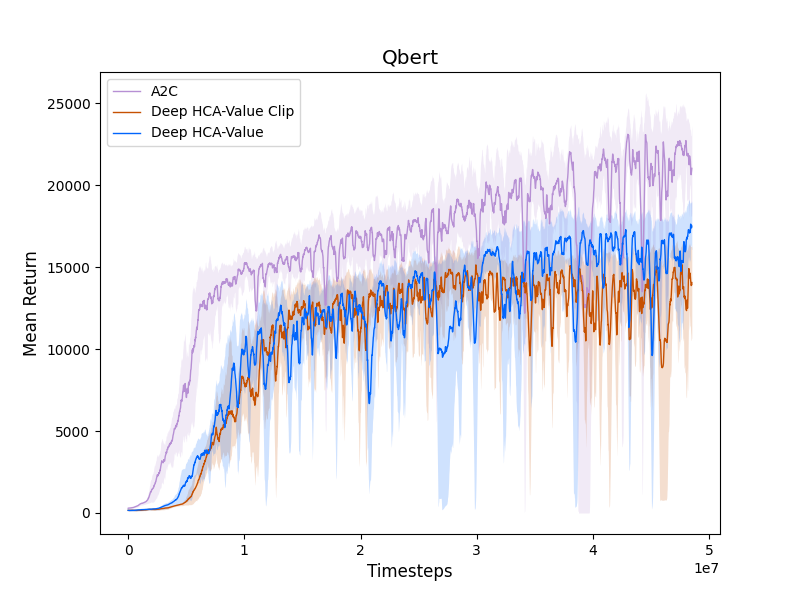}
    \includegraphics[width=0.24\linewidth]{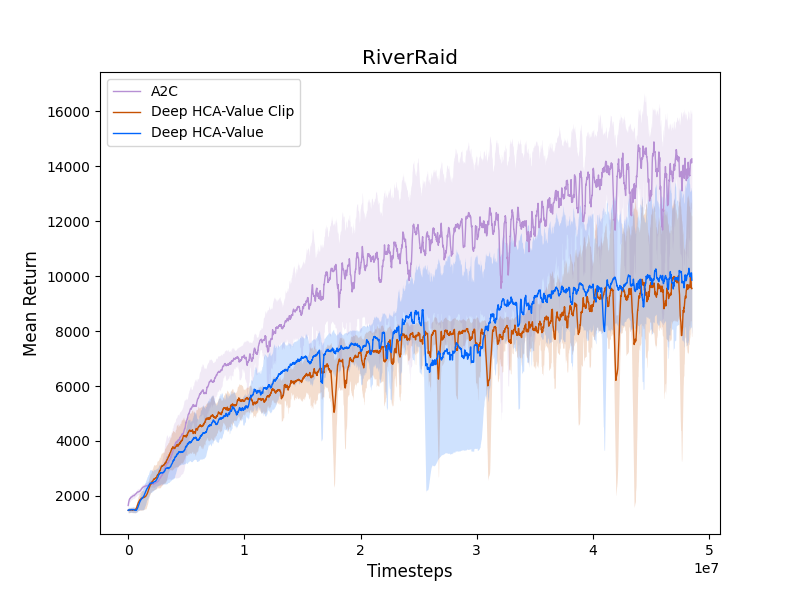}\vspace{-0.3cm}
    \includegraphics[width=0.24\linewidth]{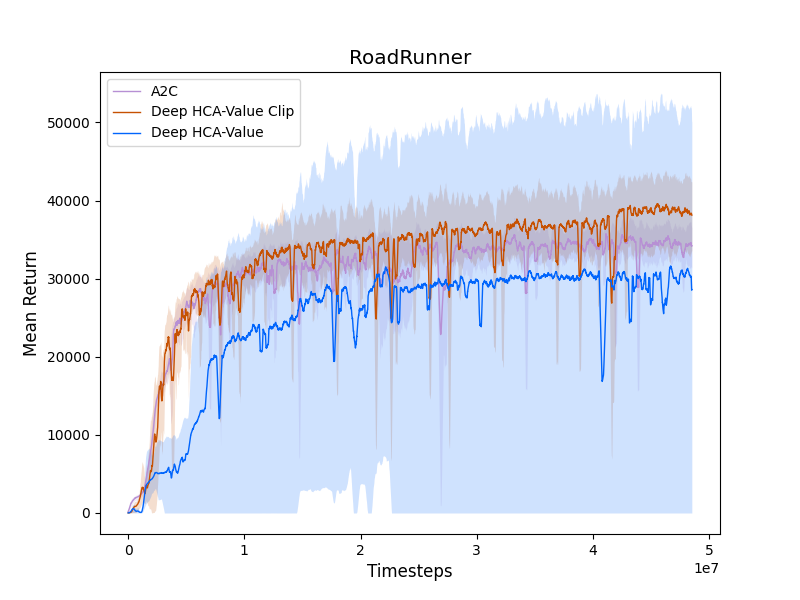}
    \includegraphics[width=0.24\linewidth]{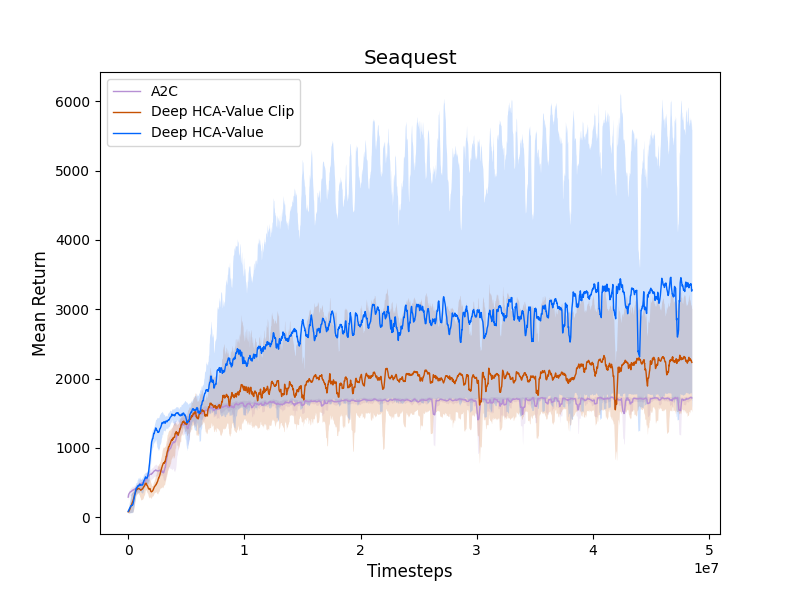}
    \includegraphics[width=0.24\linewidth]{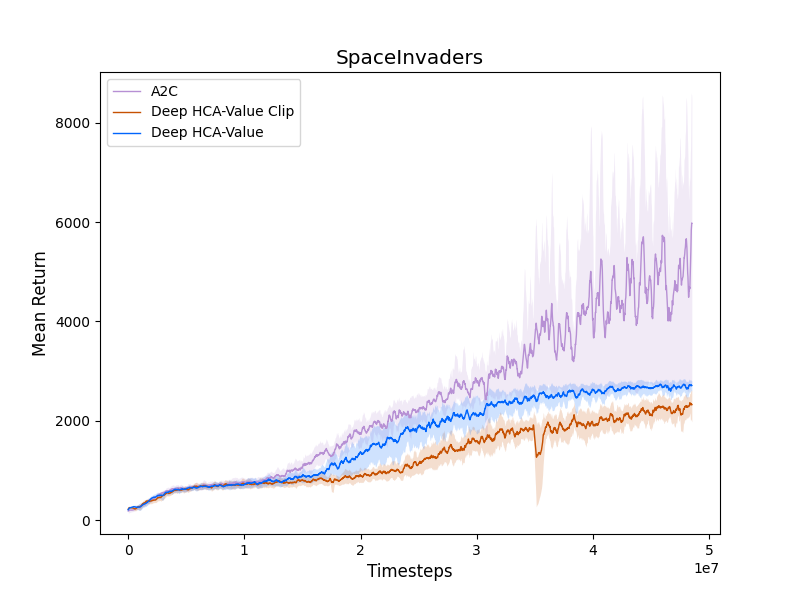}
    \includegraphics[width=0.24\linewidth]{images/icml_performance_plots/stargunner.png}\vspace{-0.3cm}
    \includegraphics[width=0.24\linewidth]{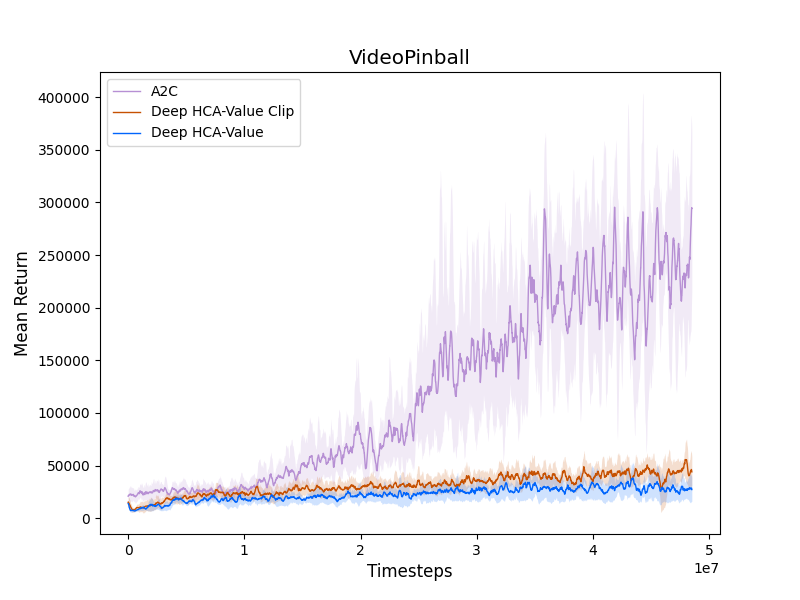}
    \includegraphics[width=0.24\linewidth]{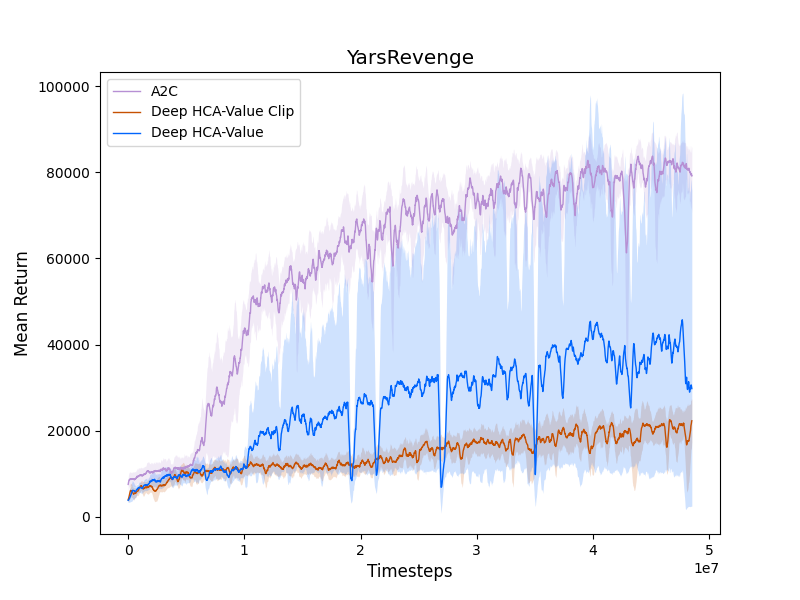}\vspace{-0.1cm}
    \caption{Performance plots on all games for Deep HCA-Value, Deep HCA-Value Clip, and A2C.}
    \label{fig:atari_training_curves}
\end{figure}

\subsection{Additional N-Step Truncation Comparisons}
\label{app:n_step_full}
In Figure \ref{fig:n_step_full} we show the five games where we compared Deep HCA-Value to our N-step truncated A2C variant. In the four games where Deep HCA-Value performs well (BeamRider, Carnival, NameThisGame, and StarGunner) we note that none of the variants exceeds its performance, although in StarGunner N=5 matches it. On DemonAttack, where Deep HCA-Value underperforms A2C, the truncations not only outperform Deep HCA-Value but also outperform A2C, suggesting that while learning to truncate returns is useful on this game the overfitting issues noted in Section \ref{subsec:what_has_credit_learned} prevent Deep HCA-Value from matching the performance of the truncated variants.

\begin{figure}
    \centering
    \includegraphics[width=0.24\linewidth]{images/icml_n_step/BeamRider.png}
    \includegraphics[width=0.24\linewidth]{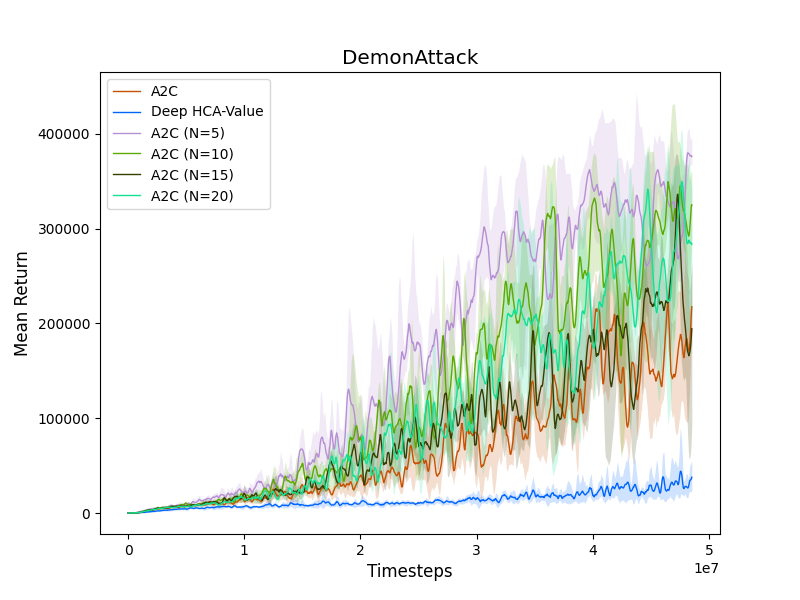}
    \includegraphics[width=0.24\linewidth]{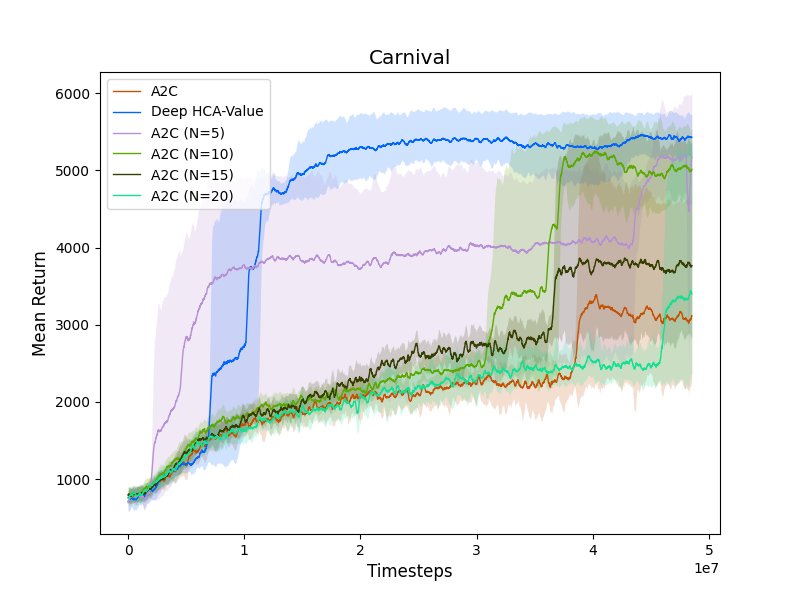}
    \includegraphics[width=0.24\linewidth]{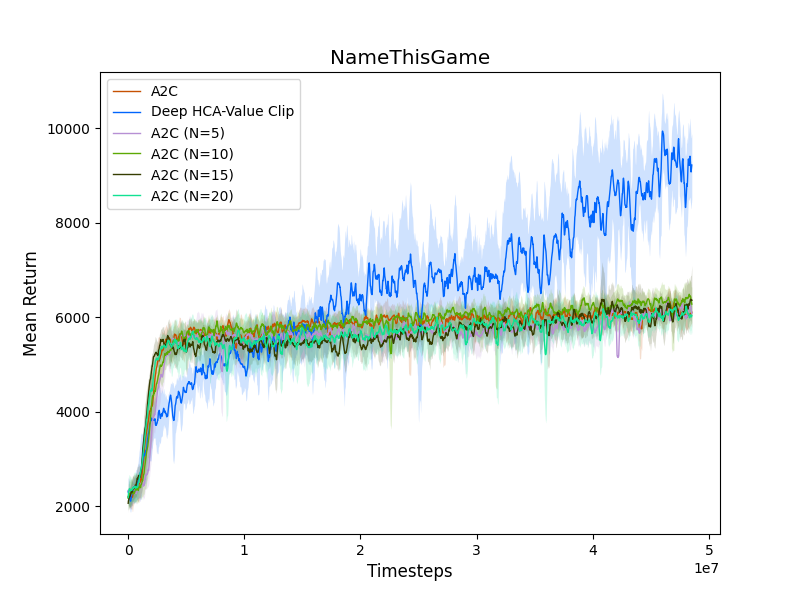}
    \includegraphics[width=0.24\linewidth]{images/icml_n_step/StarGunner.png}
    \caption{Performance plots on additional games for N-step truncation variants of A2C versus Deep HCA-Value.}
    \label{fig:n_step_full}
\end{figure}

\subsection{Additional Deep HCA-Value Ablations}
\label{app:more_ablations}

In this section, we evaluate several ablations of Deep HCA-Value. We ran Deep HCA (without our extensions), Deep HCA-Prior (with only the policy prior), Deep HCA-Prior Features (with policy logits passed as features to the hindsight classifier rather than using them as explicit prior), Deep HCA-Value with no prior (but still using bootstrapped advantages), and Deep HCA-Value using policy logits as features. The results are shown in Figure \ref{fig:hca_variants}. Other than on Breakout, Deep HCA-Value (the version described throughout this paper) performed best among the methods, though Deep HCA-Value with the policy features prior performed comparably on BeamRider.

Breakout is an interesting case because unmodified Deep HCA appears to train stably, while its extensions do not, unlike the other games tested. We hypothesize that this is because without a prior or advantages both the credit classifier and policy train slowly, which may avoid credit-induced policy collapse on this game as discussed in Section \ref{subsec:what_has_credit_learned}. Also of interest is that Deep HCA-Value without any policy prior starts out competitive with the prior+advantage methods on BeamRider and Pong, but performance collapses early in training, likely due to the moving target issues described in Section \ref{subsubsec:policy_prior}.

\begin{figure}
    \centering
    \includegraphics[width=0.24\linewidth]{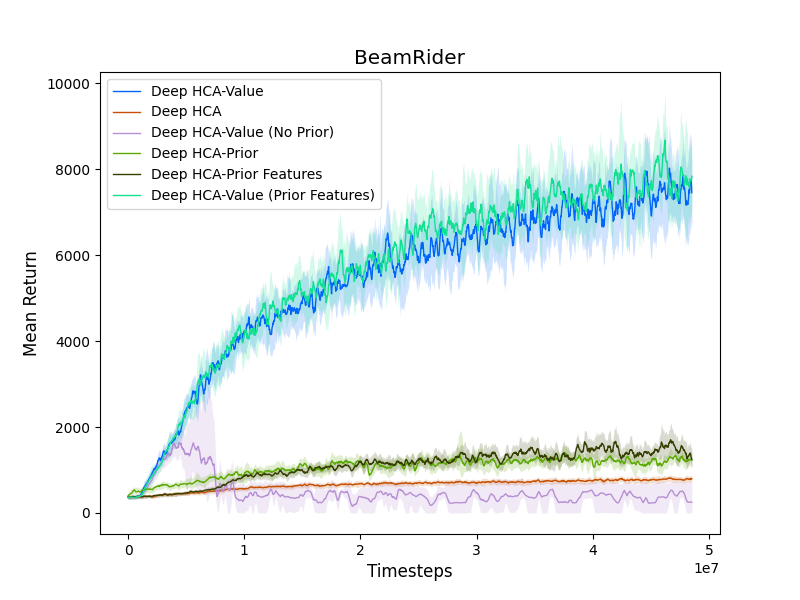}
    \includegraphics[width=0.24\linewidth]{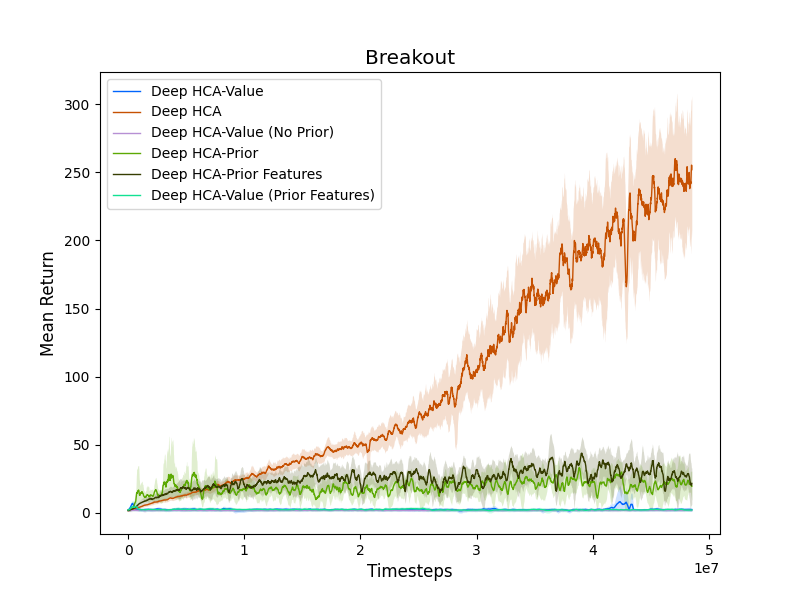}
    \includegraphics[width=0.24\linewidth]{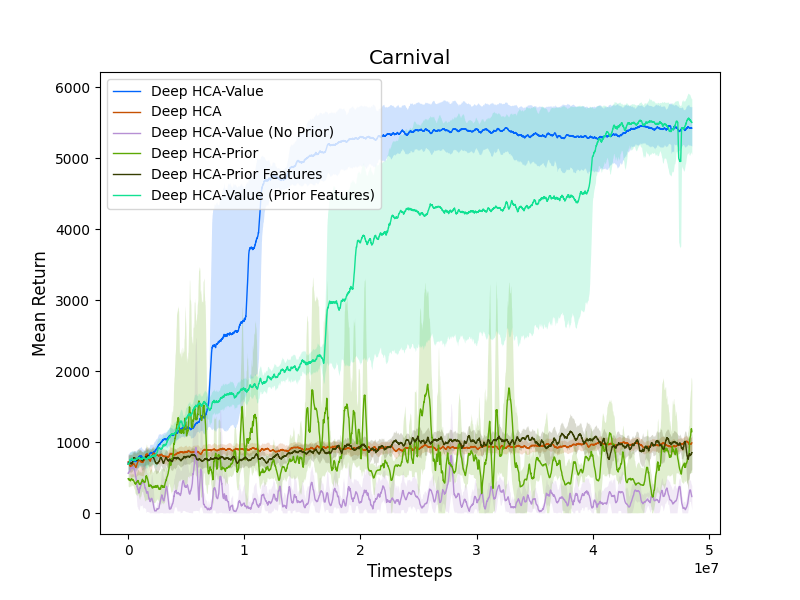}
    \includegraphics[width=0.24\linewidth]{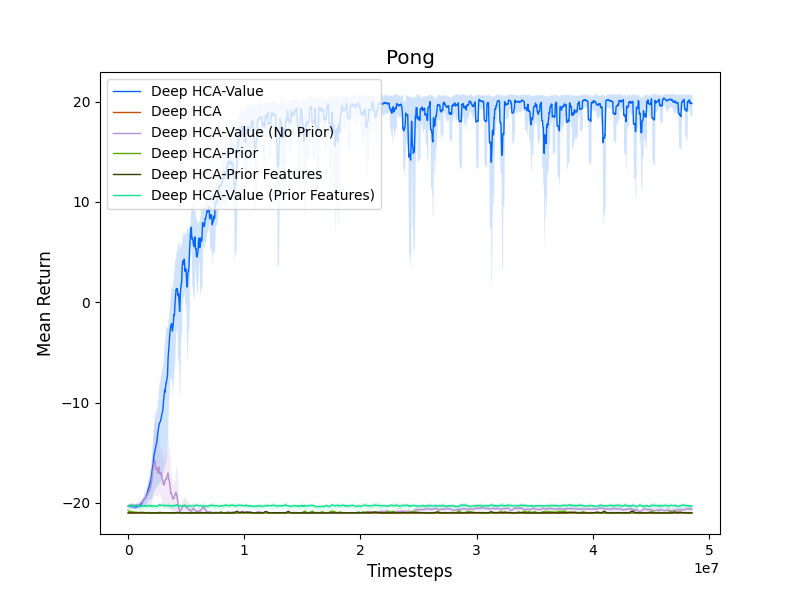}
    \includegraphics[width=0.24\linewidth]{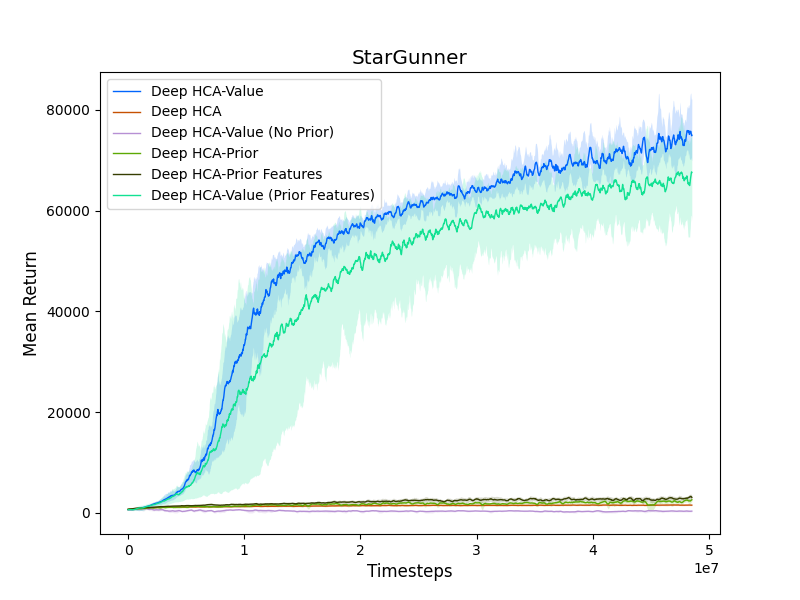}
    \caption{Performance plots of ablations to Deep HCA-Value, showing the performance of Deep HCA, Deep HCA-Prior, and various combinations on a selection of games. With the exception of Breakout (discussed further in Section \ref{sec:experiments}) Deep HCA-Value is more stable and performant than other variants.}
    \label{fig:hca_variants}
\end{figure}

\subsection{Additional Experimental Details}
\label{app:exp_details}
In this section, we further elaborate on the experimental details we used to produce our results. Our code was implemented on top of a publicly-available A2C implementation \cite{pytorchrl}, and all non-algorithm-specific hyperparameters and implementation details were shared across algorithms. The hyperparameters we used to train are shown in Table \ref{tab:hyperparams}. We use a policy and value function which share hidden layers of AtariCNN \cite{mnih2015human} network, consistent with prior work and our baseline implementation. For the credit classifier we use a separate network that takes two concatenated observations as input and otherwise is identical to AtariCNN architecture with two additional linear layers on top. The AtariCNN architecture consists of 3 convolutional layers and 1 fully-connected layer, with 32, 64, 32, and 512 channels/output units respectively, kernel sizes of 8, 4, and 3 for the 3 convolutional layers, which have strides of 4, 2, and 1. We used ReLU nonlinearities throughout. Each output head (policy, value or credit) is implemented as an additional fully-connected layer to output the desired quantity (action logits, value estimate or credit residuals respectively). We use the RMSProp optimizer to train the agent and Adam to train the credit classifier.

\begin{table}[t]
    \centering
    \begin{tabular}{|c|c|c|}
        \hline
        Name & Value & Description \\\hline
        \textbf{Shared Hyperparams} & & \\\hline
        lr & 0.0007 & Learning rate for the agent (policy and value) \\\hline
        alpha & 0.99 & $\alpha$ value for the RMSProp optimizer \\\hline
        num-steps & 32 & Number of timesteps to roll out before bootstrapping \\\hline
        num-processes & 8 & Number of parallel threads/environments rolling out \\\hline
        gamma & 0.99 & Time decay discount factor $\gamma$ for rewards \\\hline
        entropy-coef & 0.01 & Policy entropy bonus for policy training \\\hline
        value-loss-coef & 0.5 & LR multiplier for value updates \\\hline
        max-grad-norm & 0.5 & Max gradient norm for parameter gradients allowed \\\hline
        num-env-steps & 50,000,000 & Number of environment steps to train for \\\hline
        \textbf{Deep HCA Variants} & & \\\hline
        teacher-num-layers & 3 & Number of FC layers on top of AtariCNN for credit classifier \\\hline
        teacher-lr & 0.00005 & Learning rate for the credit classifier trained with Adam \\\hline
        teacher-num-batches & 8 & Number of batches to train credit per policy update \\\hline
        teacher-batch-size & 528 & Batch size for training the credit classifier \\\hline
        \textbf{Deep HCA HCA-Value Clip Only} & & \\\hline
        clip-hindsight & true & Clip hindsight probabilities as discussed in Section \ref{app:neg_rewards} \\\hline
        max-ratio & 3 & $\lambda$ as discussed in Section \ref{app:neg_rewards} \\\hline
    \end{tabular}
    \caption{Hyperparameters used for training A2C and Deep HCA-Value.}
    \label{tab:hyperparams}
\end{table}


\subsection{Negative Return Divergence}
\label{app:neg_rewards}

In this section we provide additional experiments
demonstrating the susceptibility of HCA methods to divergence in the presence of negative returns and describe extra measures to alleviate this issue.

In order to demonstrate that the policy collapse issue described in Section \ref{subsubsec:advantage_credit} is not unique to Pong, we test Deep HCA on a modified version of Breakout where we add a penalty of $-5$ for dropping the ball.

In Figure \ref{fig:breakout_negative_reward} we observe that both Deep HCA and Deep HCA with policy prior (Deep HCA-Prior) suffer from policy collapse at the start of training, and make no progress on the task. We do observe that Deep HCA-Prior's entropy collapse is less severe, which reinforces the point that this issue is related to hindsight classifier training speed -- a hindsight classifier with an explicit prior parametrization trains faster and thus mitigates policy collapse to some extent.

\begin{figure}
    \centering
    \includegraphics[width=0.49\linewidth]{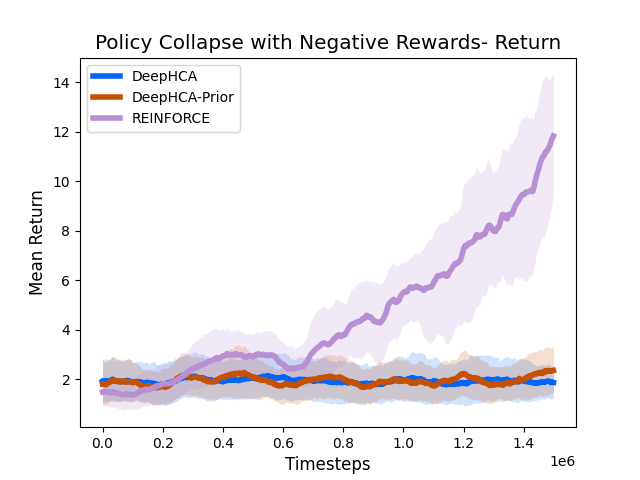} 
    \includegraphics[width=0.49\linewidth]{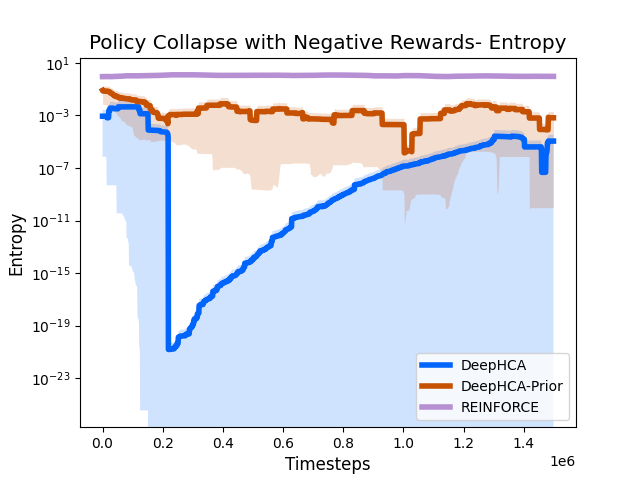}
    \caption{Training curves for Deep HCA variants without advantage alongside REINFORCE on Breakout with a reward penalty of $-5$ applied for dropping the ball. The shaded regions show min and max across 3 runs. Returns shown are computed using the original environment reward, which cannot be negative. Note that entropy is plotted on a log scale. Deep HCA without advantage suffers from policy collapse, preventing any progress on the task.}
    \label{fig:breakout_negative_reward}
\end{figure}

\subsection{Unified Crediting of Rewards}
\label{app:credit_all_rewards}
In this section, we will justify the update rule defined in Equation \eqref{eq:deep_hca_update}, specifically under which assumptions we can simplify Equation \eqref{eq:hca_update} by unifying how credit is assigned to rewards.

\begin{theorem}
\label{th:r_next_state}
Assume a MDP $(\mathcal{S}, \mathcal{A}, p, r, \gamma)$ where for any action $a$, state $s_t$ and next state $s_{t+1}\sim p(\cdot|s_t,a)$ reward is a function of next state $r=r(s_{t+1})$. For any two states $s_t$ and $s_k$ such that $k>t$ and any action $a$ let $h(a|s_t,s_k)$ be conditional probability $P(A_t=a|S_t=s_t,S_k=s_k)$ over trajectories sampled from policy $\pi_\theta$. Then the gradient of the value function at some state $s_0$ is:
\begin{equation*}
    \nabla_\theta V^{\pi_\theta}(s_0)
     = E_{\tau|s_0}\left[ \sum_{t\geq0} \gamma^t\sum_a \nabla_\theta\log\pi_\theta(a|S_t)
        \sum_{k\geq t}\gamma^{k-t}h(a|S_t,S_{k+1})R_k
    \right]
\end{equation*}
\end{theorem}
\begin{proof}
From the Policy Gradient theorem \cite{sutton1999policy} we have
\begin{equation*}
    \nabla_\theta V^{\pi_\theta}(s_0)
     = E_{\tau|s_0}\left[ \sum_{t\geq0} \gamma^t\sum_a \nabla_\theta\pi_\theta(a|S_t) Q(S_t,a)\right],
\end{equation*}
where $Q(s_t,a)\defeq E_{\tau|s_t,a}[G_t]=E_{\tau|s_t,a}\left[\sum_{k\geq t}\gamma^{k-t} R_k\right]$.

Let's express the $Q$-function in terms of probability $h(a|s_t,s_k)$:
\begin{equation*}
    \begin{split}
        Q(s_t,a) & \stackrel{\text{(a)}}{=} \sum_{k\geq t}\sum_{s\in\mathcal{X}} \gamma^{k-t} P(S_{k+1}=s|A_t=a,S_t=s_t)r(s)\\
        & \stackrel{\text{(b)}}{=} \sum_{k\geq t}\sum_{s\in\mathcal{X}} \gamma^{k-t}\frac{P(A_t=a|S_{k+1}=s,S_t=s_t)P(S_{k+1}=s|S_t=s_t)}{P(A_t=a|S_t=s_t)}r(s)\\
        & = \sum_{k\geq t}\sum_{s\in\mathcal{X}} \gamma^{k-t}P(S_{k+1}=s|S_t=s_t)\frac{h(a|s_t,s)}{\pi_\theta(a|s_t)}r(s)\\
        & = E_{\tau|s_t}\left[\sum_{k\geq t}\gamma^{k-t}\frac{h(a|S_t,S_{k+1})}{\pi_\theta(a|S_t)}R_k\right],
    \end{split}
\end{equation*}
where (a) follows from the definition of $Q(s_t,a)$ and the assumption $r=r(s_{t+1})$ and (b) follows from Bayes' rule.

We can then plug the expression above into the Policy Gradient theorem and obtain:
\begin{equation*}
    \begin{split}
    \nabla_\theta V^{\pi_\theta}(s_0)
     & = E_{\tau|s_0}\left[ \sum_{t\geq0} \gamma^t\sum_a \nabla_\theta\pi_\theta(a|S_t) E_{\tau|S_t}\left[\sum_{k\geq t}\gamma^{k-t}\frac{h(a|S_t,S_{k+1})}{\pi_\theta(a|S_t)}R_k\right]\right] \\
     & = E_{\tau|s_0}\left[ \sum_{t\geq0} \gamma^t E_{\tau|S_t}\left[\sum_a \nabla_\theta\pi_\theta(a|S_t) \sum_{k\geq t}\gamma^{k-t}\frac{h(a|S_t,S_{k+1})}{\pi_\theta(a|S_t)}R_k\right]\right] \\
     & = E_{\tau|s_0}\left[ \sum_{t\geq0} \gamma^t E_{\tau|S_t}\left[\sum_a \nabla_\theta\log\pi_\theta(a|S_t) \sum_{k\geq t}\gamma^{k-t}h(a|S_t,S_{k+1})R_k\right]\right] \\
     & = E_{\tau|s_0}\left[ \sum_{t\geq0} \gamma^t \sum_a \nabla_\theta\log\pi_\theta(a|S_t) \sum_{k\geq t}\gamma^{k-t}h(a|S_t,S_{k+1})R_k\right],
    \end{split}
\end{equation*}
where the last equality comes from applying iterated expectations.
\end{proof}

The assumption $r=r(s_{t+1})$ holds in Atari games where obtained reward is visualized for the player (via special animation or a score visible on screen), and thus each reward has a corresponding state. This assumption also holds for many other applications of interest, such as robotics, where specifying a rewarding goal state is common.

In the most general case when reward is a function of transition tuple (state, action, next state), we can obtain a similar expression for policy gradient by conditioning hindsight probability on the whole transition, which effectively reduces to conditioning on reward.

\begin{theorem}
\label{th:r_general}
Assume an MDP $(\mathcal{S}, \mathcal{A}, p, r, \gamma)$ where for any action $a$, state $s_t$ and next state $s_{t+1}\sim p(\cdot|s_t,a)$ reward is a function of transition $r=r(s_t,a,s_{t+1})$. For any two states $s_t$ and $s_k$ such that $k\geq t$ and any action $a$ let $h(a|s_t,s_k,a_k,s_{k+1})$ be the conditional probability $P(A_t=a|S_t=s_t,S_k=s_k,A_k=a_k,S_{k+1}=s_{k+1})$ over trajectories sampled from policy $\pi_\theta$. Then the gradient of the value function at some state $s_0$ is:
\begin{equation*}
    \nabla_\theta V^{\pi_\theta}(s_0)
     = E_{\tau|s_0}\left[ \sum_{t\geq0} \gamma^t\sum_a \nabla_\theta\log\pi_\theta(a|S_t)
        \sum_{k\geq t}\gamma^{k-t}h(a|S_t,S_k,A_k,S_{k+1})R_k
    \right]
\end{equation*}
\end{theorem}
\begin{proof}
Analogously to Theorem \ref{th:r_next_state} we start by expressing the $Q$-function in terms of the hindsight probability $h$ using the definition of a $Q$-function and applying Bayes' rule:
\begin{equation*}
    \begin{split}
        Q(s_t,a_t) & = \sum_{k\geq t}\sum_{s\in\mathcal{X}}\sum_{a\in\mathcal{A}}\sum_{s'\in\mathcal{X}} \gamma^{k-t} P(S_k=s,A_k=a,S_{k+1}=s'|A_t=a_t,S_t=s_t)r(s,a,s')\\
        & = \sum_{k\geq t}\sum_{s\in\mathcal{X}}\sum_{a\in\mathcal{A}}\sum_{s'\in\mathcal{X}} \gamma^{k-t}P(S_k=s,A_k=a,S_{k+1}=s'|S_t=s_t)\frac{h(a_t|s_t,s,a,s')}{\pi_\theta(a_t|s_t)}r(s,a,s')\\
        & = E_{\tau|s_t}\left[\sum_{k\geq t}\gamma^{k-t}\frac{h(a_t|S_t,S_k,A_k,S_{k+1})}{\pi_\theta(a_t|s_t)}R_k\right]
    \end{split}
\end{equation*}

We then plug the obtained expression for the $Q$-function into the Policy Gradient theorem and complete the proof analogously to Theorem \ref{th:r_next_state}.
\end{proof}

Theorem \ref{th:r_general} provides a unified way to credit every reward using hindsight probability without the need for special treatment of immediate rewards. This motivates a simplified practical implementation of hindsight-style credit without a learned immediate reward model.

\subsection{What types of games does Deep HCA-Value perform well on?}
\label{subsec:atari_results_explained}

Let's look at the games where Deep HCA-Value outperformed A2C (See Figure \ref{fig:atari_training_curves}). Qualitatively, these games share a number of common elements--
They are mostly (BeamRider, StarGunner, Carnival, SeaQuest, etc) ``shooting gallery'' type games where the agent must move and shoot at multiple targets, with delayed rewards due to the travel time of the agent's bullets. Some games where Deep HCA-Value moderately outperforms, such as Frostbite and RoadRunner, aren't shooting galleries but do require balancing short term rewards (jumping to a new platform in Frostbite) and long term rewards (Surviving to return to shore to make additional excursions) and thus have overlapping credit intervals. Curiously, games where Deep HCA-Value underperforms A2C or is comparable tend to have very long delay periods (Pong) or very short ones (Kangaroo).

Two of the games where Deep HCA-Value trains very slowly or requires clipping, Breakout and DemonAttack, are particularly interesting. Both feature pure black backgrounds and simple 1D controls, with very few reward-informative pixels (a tiny ball and mostly static wall, tiny bullets and small enemies). Notably, DemonAttack is a shooting gallery type game, which we would expect Deep HCA-Value to do well on given other games of the type.

Based on the evidence in Section \ref{subsec:what_has_credit_learned}, we know that severe mispredictions degrade performance in both DemonAttack and (especially) Breakout. We hypothesize that this is in both cases due to the classifier learning to focus on the player's avatar (paddle or laser cannon) as the most informative feature for predicting actions. As this avatar is directly controlled by the agent, it is easy to predict which action moves it to a given location. We suspect that the credit classifier overfits to assign credit to actions that move the avatar to its future reward-correlated position (ignoring reward-relevant features) rather than crediting actions that result in the entire observed rewarding state (including signifiers of reward).

Of the games tested, Deep HCA-Value failed to train (with or without clipping) on two games where A2C learned a non-random policy: Asteroids and Atlantis. Asteroids features a complex action space and visually diverse states, which likely make credit classifier learning difficult, while Atlantis has very rare rewards with very long delay periods, which likewise make credit hard to learn. As the credit classifier must learn something before the policy can improve, this result fits the hypothesis that these games represent particularly hard cases for credit assignment.

\end{document}